\documentclass{article}

\usepackage{fullpage}

\usepackage{algorithm}
\usepackage{mathtools,amsmath,amssymb,amstext,amsthm,tikz,thmtools}
\usepackage[inline]{enumitem}

\usepackage[utf8]{inputenc} 
\usepackage[T1]{fontenc}    
\usepackage{hyperref}       
\usepackage{url}            
\usepackage{booktabs}       
\usepackage{amsfonts}       
\usepackage{nicefrac}       
\usepackage{microtype}      

\newcommand{\mc}{\mathcal}
\newcommand{\mb}{\mathbf}
\newcommand{\tr}{\text{Tr}}
\newcommand{\diag}{\text{diag}}
\newcommand\numberthis{\addtocounter{equation}{1}\tag{\theequation}}

\DeclareMathOperator*{\argmin}{arg\,min}

\newtheorem{theorem}{Theorem}
\newtheorem{lemma}[theorem]{Lemma}
\newtheorem{definition}[theorem]{Definition}

\newtheorem{corollary}[theorem]{Corollary}

\title{Provably noise-robust $k$-means clustering}

\author{\normalsize
{Shrinu Kushagra} {\textnormal {,}} {Yaoliang Yu} {\textnormal {and}} {Shai Ben-David} \\
\normalsize David R. Cheriton School of Computer Science \\
\normalsize University of Waterloo,\\
\normalsize Waterloo, Ontario, Canada\\
\normalsize \{skushagr,yaoliang.yu,shai\}@uwaterloo.ca \\
}

\begin{document}

\maketitle

\begin{abstract}
We consider the problem of clustering in the presence of noise. That is, when on top of cluster structure, the data also contains a subset of \emph{unstructured} points. Our goal is to detect the clusters despite the presence of many unstructured points. Any algorithm that achieves this goal is noise-robust. We consider a regularisation method which converts any center-based clustering objective into a noise-robust one. We focus on the $k$-means objective and we prove that the regularised version of $k$-means is NP-Hard even for $k=1$. We consider two algorithms based on the convex (sdp and lp) relaxation of the regularised objective and prove robustness guarantees for both. 

The sdp and lp relaxation of the standard (non-regularised) $k$-means objective has been previously studied by \cite{awasthi2015relax}. Under the stochastic ball model of the data they show that the sdp-based algorithm recovers the underlying structure as long as the balls are separated by $\delta > 2\sqrt{2} + \epsilon$. We improve upon this result in two ways. First, we show recovery even for $\delta > 2 + \epsilon$. Second, our regularised algorithm recovers the balls even in the presence of noise so long as the number of noisy points is not too large. We complement our theoretical analysis with simulations and analyse the effect of various parameters like regularization constant, noise-level etc. on the performance of our algorithm. In the presence of noise, our algorithm performs better than $k$-means++ on MNIST. 
\end{abstract}

\section{Introduction}
Clustering aims to group similar data instances together while separating dissimilar ones. However, often many datasets have, on top of cohesive groups, a subset of ``unstructured'' points as well. In such cases, the goal is to detect the cohesive structure while simultaneously separating the unstructured data points. Clustering in such situations can be viewed as a {\em noise-robustness} problem. 

Another important issue is that the clustering problem is {\em under-specified}. This means that the same dataset might need to be clustered in different ways depending upon the intended application. Consider, for example, the problem of clustering users of a movie-streaming service such as Netflix. The output clustering can be used to suggest similar movies to similar users or to gain insights into the daily/monthly behaviour of the users. Depending on the application, different clustering algorithms need to be chosen. Hence, any solution to clustering challenges like noise-robustness, under-specificity should be such that it is applicable across a wide-range of clustering algorithms. 

We propose a general method of regularisation that transforms any clustering objective which outputs $k$ clusters to one that outputs $k+1$ clusters. The algorithm is now allowed to `discard' a bunch of points into the extra `garbage' or noise cluster by paying a constant regularization penalty. The intuition is that this will make it easier to detect the structure in the remaining points. Similar to regularisation, {\em noise prototypes} (points which are equidistant to all other points) were considered by \cite{dave1993robust}. However, that idea was used only in the limited context of Lloyd’s algorithm and without any theoretical or noise robustness guarantees.

In this paper, we consider the following framework motivated by \cite{ben2014clustering}. We are given an input dataset $\mc X$ made of two components. The first is the clusterable or ``nice'' subset $\mc I$ which is the union of $k$ unit balls $B_i$ separated by a distance of atleast $\delta$. The second is the unstructured noise component $\mc N$. Note that the clustering algorithm only sees $\mc X$ and is not aware in advance of $\mc I$ or $\mc N$. The addition of $\mc N$ makes it more difficult to detect the structure in $\mc I$. Consider the original algorithm $\mc A$ and its robustified transformation $\mc A' = \mc R_{\lambda}(\mc A)$. $\mc A'$ is obtained by using our regularisation paradigm and specifying a parameter $\lambda$. By comparing the clusterings $\mc A(\mc I)$ and the clustering $\mc A'(\mc X)$ restricted to $\mc I$, we can examine the effect of $\mc N$ (in terms of size and distance relative to $\mc I$) on the ability to detect the cohesive structure of $\mc I$.

In this work, we consider two choices for $\mc A$. An algorithm based on SDP relaxation of the $k$-means objective and another based on LP relaxation. \cite{awasthi2015relax} showed that for $\delta > 2\sqrt{2}(1+\frac{1}{\sqrt d})$ (where $d$ is the dimension of the euclidean space) the SDP-based algorithm recovers the clustering of $\mc I$ if the balls $B_i$ are generated by an isotropic distribution (stochastic ball model). \cite{iguchi2015tightness} `improved' this to $\delta > 2 + \frac{k^2}{d}cond(\mc I)$. However, the condition number (ratio of maximum distance between any two centers and the minimum distance between any two cluster centers) can be arbitrarily large. We improve this to $\delta > 2\big(1+\sqrt\frac{k}{d}\big)$ which is optimal for large $d$. To our knowledge, this is best known guarantee for the SDP-based $k$-means algorithm. Note that all the above results are for the noiseless case. For the noisy case, our robustified (or regularised) version recovers the clustering of $\mc I$ for $\delta > 2\big(1+\sqrt{\sigma + \frac{k}{d}}\big)$ where $\sigma$ is a term which depends on the ratio of number of noisy points and the number of points in the smallest cluster. 

We also consider the distribution-free setting where  the balls $B_i$ have been generated by any unknown distribution. In this setting, the separation requirement for recovering the structure of $\mc I$ becomes $\delta > 2(1 + \sqrt{k})$ for the noiseless and $\delta > 2(1+\sqrt{\sigma + k})$ for the noisy case. As before, $\sigma$ depends on the number of noisy points. We also give robustness guarantees for the LP based algorithm. However, the separation requirement in this case is more strict $\delta > 4$ and the requirements on noise are also stronger (noisy points from the set $\mc N$ should be far from $\mc I$).  

We also prove hardness results for the regularised $k$-means objective. For $k \ge 2$, the NP-hardness follows from the NP-Hardness of the standard (non-regularised) $k$-means objective. We show that the regularised objective is NP-Hard to optimize even for $k=1$. An important choice in the implementation of the regularised algorithm is the value of the regularization constant $\lambda$. For $\lambda = 0$, the problem becomes trivial and for $\lambda \rightarrow \infty$, the problem reduces to $k$-means. We prove that there exists a range of $\lambda$ depending on the sizes of the clusters and the separation $\delta$ between them such that the regularised algorithm recovers the underlying structure when given that $\lambda$ as input. We also conduct simulation studies where we examine the effect of $\lambda$, the number of noisy points $m$, the separation $\delta$ and other parameters on the performance of our regularised SDP-based algorithm. We also perform experiments on the MNIST dataset. We observed that the regularised version performed better than $k$-means++ when the dataset had outliers. In the absence of outliers, the performance of both these algorithms were similar. 

\subsection{Related Work}
The problem of clustering in the presence of noise has been studied before both in distribution-based and distribution-free settings. In the distribution-based setting, the goal is to estimate the parameters of the distribution (say the mean and variance of gaussian etc). In the distribution-free setting, the goal is to prove that if the data has some structure (is clusterable) then the (proposed or existing) clustering algorithm recovers that structure even in the presence of noise. Different works define different notions of `clusterable' data. In the current work, the separation requirement on the clusters was global. That is, the each cluster was separated by atleast $\delta$ times the maximum radius amongst all the clusters. Other notions of clusterability, like $\alpha$-center proximity \cite{awasthi2012center}or $\gamma$-margin \cite{ashtiani2016clustering} require that two clusters be separated relative to their radii.

Different works on clustering have also made different assumptions on the type of noisy points. The most common is to assume that the noise is adversarial but the number of adversaries is not too large. For example, \cite{balcan2012clustering} and \cite{balcan2008discriminative} provide bounds for clustering in the presence of noise as long as the number of adversaries is constant-factor smaller compared to the size of the smallest cluster. \cite{kushagra2016finding} considered noise which is structureless, that is the noisy points do not form dense large subsets. Another field of work is to address noisy part of the data as being generated by some uniform random noise or gaussian perturbations \cite{cuesta1997trimmed}, \cite{garcia2008general} and \cite{dave1993robust}.  

Another line of work which is related to ours is clustering a mixture of $k$ gaussians with few adversaries. The best known result is by \cite{awasthi2012improved} which requires that the mean of the gaussians be separated by $\tilde O(\sigma \sqrt{k})$ where $\sigma^2$ is an upper bound on the variance of the $k$ gaussians. Recently, \cite{charikar2017learning} matched this result using different techniques. Although the distribution free setting considered in this paper is different from the above works, the separation required for the SDP-based algorithm to succeed also has a similar dependence on $k$, namely $\delta > 2(1 + \sqrt{k})$. 

Some works examine the robustness of different algorithms when the number of clusters is the same for the original data and the data with added noise. They show that in this setting the traditional algorithms are provably not noise-robust \cite{hennig2008dissolution} and \cite{ackerman2013clustering}. Another line of work which is related to ours is based on the convex relaxation of center-based clustering objectives. \cite{peng2007approximating} was the first to formulate the $k$-means cost function as a 0-1 SDP and then subsequently relaxed it to a standard SDP. In this work, we use a similar technique to first obtain a 0-1 SDP and subsequently a relaxed SDP for the regularised $k$-means objective. 

\section{Preliminaries and definition}
\label{sec:pre}

Let $(\mb M, d)$ be a metric space. Given a finite set $\mc X \subset \mb M$, a $k$-clustering $\mc C$ of $\mc X$ partitions the set into $k$ disjoint subsets $\mc C = \{C_1, \ldots, C_k\}$. An objective-based clustering algorithm associates a cost with each possible partition of $\mc X$ and then tries to find the clustering with minimum cost. Throughout this section, $f$ denotes a function on the nonnegative reals.

\begin{definition}[$(k, f)$-objective algorithm] Given $\mc X\subset \mb M$ and a distance function $d$, a $(k, f)$-objective based algorithm $\mc A$ tries to find centers $\mu_1, \ldots, \mu_k \in \mb M$ so as to minimize the following function
\begin{align}
\label{eqn:kfObjectiveAlg}
\text{Cost}(\mu_1, \ldots, \mu_k) = \sum_{x\in \mc X} f(d(x, \mu(x))), \quad \mu(x) = \argmin_{\mu \in \{\mu_1, \ldots, \mu_k\}} d(x, \mu).
\end{align}
\end{definition}
Note that algorithm $\mc A$ may not often find the optimal solution because for many common functions $f$, solving the optimization is NP-Hard. Thus, heuristics are used that can get stuck at a local minima. For example, when $f(x) = x^2$, the above definition corresponds to the $k$-means objective, and the Lloyd's algorithm that is used to solve this objective can get stuck at a local minima. 

\begin{definition}[$(k, f)$-$\lambda$-regularised objective algorithm] Given $\mc X\subset \mb M$ and a distance function $d$, a $(k, f)$-$\lambda$-regularised objective based algorithm $\mc A'$ tries to find centers $\mu_1, \ldots, \mu_k \in \mb M$ and set $\mc I\subseteq \mc X$ so as to minimize the following function
\begin{align}
\label{eqn:kfObjectiveAlg}
\text{Cost}(\mu_1, \ldots, \mu_k, \mc I) = \sum_{x\in \mc I} f(d(x, \mu(x))) + \lambda |\mc X \setminus \mc I|,
\end{align}
where $\mu(x) = \argmin_{\mu_i} d(x, \mu_i).$
\end{definition}

The regularised objective allows discarding certain points into a ``garbage'' cluster at the expense of paying a constant penalty. The intuition is that this will help the algorithm better detect the structure of the remaining points. We will see in \S \ref{section:hardness} that minimizing this objective function is NP-Hard for all $k \ge 1$.

\subsection{Robustification paradigms}
\begin{definition}[$\lambda$-Regularised Paradigm] The $\lambda$-regularised paradigm is a robustification paradigm which takes as input a $(k, f)$-objective algorithm $\mc A$ and returns a $(k, f)$-$\lambda$-regularised objective algorithm $\mc A'$.  
\end{definition}

In this work, we focus on robustification of the $k$-means objective. Hence, it is useful to define the regularised $k$-means objective as we will refer to it many times in the remainder of the paper. 

\subsection{Regularised $k$-means objective}
Given a finite set $\mc X \subset \mb R^{d}$ and an integer $k$, the regularised $k$-means objective aims to partition the data into $k+1$ clusters $\mc C = \{C_1, \ldots, C_{k}, C_{k+1}\}$ so as to solve 
\begin{align}
	\min_{C_1,\ldots, C_{k+1} \atop c_1, \ldots, c_k} \sum_{i=1}^k \sum_{x \in C_i} \|x-c_i\|_2^2 + \lambda |C_{k+1}|.
	\label{eqn:modifiedkmeans}
\end{align}

Note that the first term of the objective depends on the $l_2$ norm, while the second term depends only on the cardinality of the ``noise'' cluster. In order to make our objective function invariant to scaling, the regularization constant $\lambda$ is added to the cost function. 

Let $m(\mc X) := \min_{x\ne y \in \mc X} \|x-y\|_2^2$ and $N = |\mc X|$. Then, it is easy to see when $\lambda \leq \frac{m(\mc X)}{2}$, then \eqref{eqn:modifiedkmeans} admits a trivial solution: each cluster $C_i$ for $i \leq k$ has exactly one point and all the remaining points are in $C_{k+1}$, leading to an objective $\lambda (N-k)$. Indeed, for any other clustering with $|C_i| = n_i$ its objective is at least $\sum_{i=1}^k (n_i-1) m(\mc X)/2 + \lambda n_{k+1} = (n-n_{k+1} - k) m(\mc X)/2 + \lambda n_{k+1}$, where we have used the simple fact:
\begin{align}
\forall C \subset \mc X, ~~
\min_{c} \sum_{x \in C} \|x - c\|_2^2 ~=~ \frac{1}{2|C|}\sum_{x, y \in C} \|x - y\|_2^2.
\end{align}
Comparing the objectives we see the solution is indeed trivial when $\lambda \leq m(\mc X) /2$. Surprisingly, for the interesting case when $\lambda > m(\mc X) /2$, the problem suddenly becomes NP-Hard, as we prove below.

\subsection{Robustness measure}
Given two clusterings $\mc C$ and $\mc C'$ of the same set $\mc X$, we define the distance between them, $\Delta(\mc C, \mc C')$, as the fraction of pairs of points which are clustered differently in $\mc C$ than in $\mc C'$. 
Given $\mc I \subseteq \mc X$, $\mc C|\mc I$ denotes the restriction of the clustering $\mc C$ to the set $\mc I$.

\begin{definition}[$\gamma$-robust \cite{ben2014clustering}]
Given $\mc X\subset \mb M$ and clustering algorithm $\mc A$, let $\mc A'$ be its robustified version obtained using any robustification paradigm. Given $\mc I \subseteq \mc X$, we say that $\mc I$ is $\gamma$-robust w.r.t $\mc X \setminus \mc I$ and $\mc A'$ if 
\begin{align}
\Delta(\mc A'(\mc X)|\mc I, \mc A(\mc I)) \le \gamma
\end{align}
\end{definition}
This measure tries to quantify the difference in the clustering of the set $\mc I$ after the addition of `noisy' points $\mc X \setminus \mc I$. If $\mc A'$ is indeed robust to noisy points, then the clusterings should be similar.

\section{Hardness of regularised $k$-means}
\label{section:hardness}

In this section, we present hardness results for the regularised $k$-means objective. 
The proof for $k \ge 2$ is fairly straightforward and follows from known hardness results for the standard (non-regularised) $k$-means. The more interesting case is when $k = 1$. It is well-known that $1$-means can be solved in linear time \cite{Bellman73} hence the same reduction does not work any more. We reduce an instance of the MAX-CLIQUE problem to the regularised $1$-means problem. We give a proof sketch for the reader's intuition. The technical details can be found in the supplementary section.


\begin{restatable}{theorem}{hardForkone}
\label{theorem:hardFork1}
Given a clustering instance $\mc X \subset \mb R^d$. Finding the optimal solution to the regularised $1$-means objective is NP-Hard for all $\lambda > m(\mc X)/2$, where recall that $m(\mc X) := \min_{x\ne y \in \mc X} \|x-y\|_2^2$.
\end{restatable}

\begin{proof}[Proof sketch]
The proof has two parts. We first show that for fixed $\lambda$ the problem is NP-Hard. The proof works by reducing an instance of MAX-CLIQUE to the regularised $1$-mean instance. The idea is to define the distance between any pair of vertices as $1$ if there exists an edge between them. If not, then define the distance as $1 + \Delta$ for a suitably chosen $\Delta$. This construction guarentees that the problem is NP-Hard for atleast one $\lambda > \frac{m(\mc X)}{2}$. Next, using a scaling argument we show that if the problem is NP-Hard for one particular $\lambda$, then it is NP-Hard for all $\lambda > \frac{m(\mc X)}{2}$.
\end{proof}

The above theorem infact shows that regularised $k$-means is hard for all $k \ge 1$. This is becuase we can reduce an instance of regularised $1$-means to regularised $k$-means by placing $k-1$ points very far away. 

\section{The regularised $k$-means SDP-based algorithm}
\label{section:heuristic}
In the previous section, we showed that the regularised $k$-means objective is NP-Hard to optimize. Hence, we cannot hope to solve the problem exactly unless $P=NP$. In this section, we develop an algorithm based on semi-definite programming relaxation of the regularised objective. 

\cite{peng2007approximating} developed an algorithm $\mc A$ (Alg. \ref{alg:heuristicSDP} with $\lambda = \infty$) which tries to minimize the $k$-means objective. They obtained a convex relaxation of the $k$-means objective and solved it polynomially using standard solvers. In this section, we use the same technique to obtain and efficiently solve the convex relaxation of the regularised $k$-means objective. Our algorithm $\mc A'$ (Alg. \ref{alg:heuristicSDP}) is the robustified version of $\mc A$ using the $\lambda$-regularised paradigm. In \S \ref{subsection:sdpAlg}, we give the details of how we transform the regularised objective into an SDP. \S \ref{subsection:sdpRobust} has our main results where we give robustness guarentees for $\mc A'$.    

\begin{algorithm}[t]
\caption{SDP-based regularised $k$-means algorithm}
\label{alg:heuristicSDP}
	\textbf{Input: }{ $\mc{X} \subset R^d$, $k$, and hyperparameter $\lambda$.} \\
	\textbf{Output: }{$\mc C' := \{C_1, \ldots, C_k, C_{k+1}$\}.}
	
	Compute the matrix $D_{ij} = \|x_i-x_j\|^2_2$.\\
	Solve the SDP (Eqn. \ref{eqn:regularisedSDP}) using any standard SDP solver and obtain matrix $Z$ and vector $y$.\\
	Use the rounding procedure (Alg. \ref{alg:roundSDP}) to obtain the partition $\mc C'$. 
\end{algorithm}

\subsection{The SDP-based algorithm}
\label{subsection:sdpAlg}
The SDP relaxation of the regularised $k$-means objective is obtained in two steps. Using similar technique to that of \cite{peng2007approximating}, we translate Eqn. \ref{eqn:modifiedkmeans} into a 0-1 SDP (Eqn. \ref{eqn:regularisedSDP}). We then prove that solving the 0-1 SDP exactly is equivalent to solving the regularised $k$-means problem exactly. Then, we relax some of the constraints of the 0-1 SDP to obtain a tractable SDP which we then solve using standard solvers. We then describe the rounding procedure which uses the solution of the SDP to construct a clustering of the original dataset. 

\begin{equation*}
	\begin{split}
	\textbf{0-1}\\
	\textbf{SDP}
  \end{split}
	\begin{cases}
		\min_{Z, y} \enspace &\tr(DZ) + \lambda \langle \mb 1, y\rangle\\
		\text{s.t. } \enspace &\tr(Z) = k\\
		& Z\cdot \mb 1 + y = \mb 1\\	
		&Z\ge 0, Z^2 = Z, Z^T = Z \\
		& y \in \{0, 1\}^n
	\end{cases}
	\xrightarrow{\text{relaxed}} \textbf{ SDP } 
	\begin{cases}
		\min_{Z, y} \enspace &\tr(DZ) + \lambda \langle \mb 1, y\rangle\\
        \text{s.t. } \enspace &\tr(Z) = k\\
		& \Big(\frac{Z+Z^T}{2}\Big)\cdot \mb 1 + y = \mb 1\\		
		&Z \ge 0, y \ge 0, Z \succeq 0 \numberthis\label{eqn:regularisedSDP}
	\end{cases}
\end{equation*}

\begin{restatable}{theorem}{modifiedkmeans}
\label{thm:modifiedkmeans}
Finding a solution to the 0-1 SDP (\ref{eqn:regularisedSDP}) is equivalent to finding a solution to the regularised $k$-means objective (\ref{eqn:modifiedkmeans}). 
\end{restatable}

Equation \ref{eqn:regularisedSDP} shows our 0-1 SDP formulation. The optimization is NP-Hard as it is equivalent to the regularised $k$-means objective. Hence, we consider a convex relaxation of the same. First, we replace $Z^2 = Z$ with $Z \succeq 0$. In addition, we relax $y \in \{0, 1\}^n$ to $y \geq 0$, as the constraint $y\leq 1$ is redundant. Using these relaxations, we obtain the SDP formulation for our objective function.

We solve the SDP using standard solvers \cite{yang2015sdpnal+} thereby obtaining $Z, y$. 
The proof of Thm. \ref{thm:modifiedkmeans} showed that the optimal solution of 0-1 SDP is of the following form. $Z$ is a $n\times n$ block diagonal matrix of the form $\diag(Z_{I_1}, \ldots, Z_{I_k}, 0)$, where $n = |\mc X|$ and $Z_{I_i} = \frac{1}{|C_i|}11^T$. Thus, given $Z$, we can extract the set of cluster centers $C = ZX$ which is an $n \times d$ matrix. Each row $C_i$ contains the cluster center to which data point $x_i$ belongs. For the points $x_j$ assigned to the noise cluster, the corresponding row $C_j$ is zero and $y_j = 1$. The SDP solver does not always return the optimal solution, as the relaxation is not exact. However, we expect that it returns a near-optimal solution. Hence, given $Z$ and $y$ returned by the solver, we use Alg. \ref{alg:roundSDP} to extract a clustering of our original dataset. The $threshold$ parameter indicates our confidence that a given point is noise. In our experiments, we have used a threshold of $0.5$. We did not tune the threshold as in our experiments the results
are not very sensitive to it.

\begin{algorithm}[t]
\caption{Regularised $k$-means rounding procedure}
\label{alg:roundSDP}
	\textbf{Input: }{ $Z \subset \mathbb{R}^{n\times n}$, $y \subset \mathbb{R}^{n}$, $\mc{X}$, and $threshold \in [0,1]$.}\\
	\textbf{Output: }{$\mc C'$.}
	
	If $y_i > threshold$ then
	\begin{itemize}[nolistsep] 
		\item[] Delete $z_i$ and $z_i^T$ from $Z$. Put $x_i$ in $C_{k+1}$.
		\item[] Delete $x_i$ from $X$.
	\end{itemize}
	$k$-cluster the columns of $X^TZ$ to obtain clusters $C_1, \ldots, C_k$.\\
	Output $\mc C' = \{C_1, \ldots, C_k, C_{k+1}\}$.
\end{algorithm}

\subsection{Robustness guarantees}
\label{subsection:sdpRobust}

Assume that we are given a set $\mc I$ of $k$ well-separated balls in $\mb R^d$. That is, $\mc I := \cup_{i=1}^k B_i$ where each $B_i$ is a ball of radius at most one and centered at $\mu_i$ such that $\|\mu_i - \mu_j\| \ge \delta$. On top of this structure, points are added from the set $\mc N$. Let $\mc A$ and $\mc A'$ be the SDP based standard and regularised $k$-means algorithm respectively (as defined in the begining of \S \ref{section:heuristic}). We will show that $\mc I$ is $0$-robust w.r.t $\mc A'$ and $\mc N$ under certain conditions on $\delta$ and mildness properties of the set $\mc N$. To show this, we need to compare the clusterings $\mc A(\mc I)$ and $\mc A'(\mc X)|_{\mc I}$. We first prove recovery guarentees for $\mc A$ in the absense of noisy points.

\begin{theorem}
\label{thm:SDPGeneral}
Given a clustering instance $\mc I \subset \mb R^{N\times d}$ and $k$. Let $\mc I := \cup_{i=1}^k B_i$ where $B_i$ is a finite set of radius at most one centered at $\mu_i$ and $\|\mu_i - \mu_j\| > \delta$. Let $\mc I' = \cup_i B_i'$ where $B_i' := \{x - \mu_i : x \in B_i\}$. Define $ n := \min_{i\in[k]} |B_i|$ and $\rho = \frac{N}{nk}$.
\begin{enumerate}[leftmargin=*,nolistsep,noitemsep]
	\item \textbf{Distribution-free} - If the distance between the centers of any two balls 
$$\delta > 1 + \sqrt{1+ 2\frac{\sigma_{\max}^2(\mc I')}{n}}$$ 
then the $k$-means SDP finds the intended cluster solution  $\mc C^* = \{B_1, \ldots, B_k\}$.
	\item \textbf{Stochastic ball assumption} - Let $\mc P$ denote the isotropic distribution on the unit ball centered at origin. Given points $c_1, \ldots, c_k$ such that $\|c_i - c_j\| > \delta  > 2$. Let $\mc P_i$ be the measure $\mc P$ translated with respect to the center $c_i$. If each $B_i$ is drawn i.i.d w.r.t the distribution $\mc P_i$ and 
$$\delta > 1 + \sqrt{1+\frac{2\theta\rho k}{d}\Big(1+\frac{1}{\log N}\Big)^2}$$  
where $\theta = \mb E[\|x_{pi}-c_p\|^2] < 1$, then there exists a constant $c > 0$ such that with probability at least $1 - 2d\exp(\frac{-cN\theta}{d\log^2N})$ the $k$-means SDP finds the intended cluster solution  $\mc C^* = \{B_1, \ldots, B_k\}$.
\end{enumerate}
\end{theorem}

Thm. \ref{thm:SDPGeneral} improves the result of Thm. 11 in \cite{awasthi2015relax}. Under the stochastic ball assumption, they showed that the $k$-means SDP finds the intended solution for $\delta > 2\sqrt{2}(1+\frac{1}{\sqrt{d}})$. For $k \ll d$, which is the case in many situations our bounds are optimal in terms of the separation reqirement of the clusters.  \cite{iguchi2015tightness} obtained similar results but for $\delta > 2 + \frac{k^2}{d}cond(\mc I)$. Asymptotically, as $d$ goes to $\infty$ their bound matches our result. However, the condition number (ratio of maximum distance between any two centers to the minimum distance between any two centers) can be arbitrarily large. We also have a better dependence on $k$ and $d$. 

Next, we analyse the recovery guarentees for $\mc A'$ in the presence of noisy points $\mc N$. We decompose the noisy points into two disjoint sets $\mc N_1$ and $\mc N_2$. The set $\mc N_2$ consists of all the points which are far from any of the points in $\mc I$. The set $\mc N_1$ consists of points which are close to atleast one of the clusters. We also require that any point in $\mc N_1$ has an $\alpha$-margin w.r.t to the centers of the balls $B_1, \ldots, B_k$. That is the difference of the distance between any point in $\mc N_1$ to a cluster center is atleast $\alpha$. Now, we will show that if $\mc N$ has the aforementioned properties then $\mc I$ is robust w.r.t the regularised SDP algorithm $\mc A'$.

\begin{theorem}
\label{thm:regularisedSDPGeneral}
Given a clustering instance $\mc X \subset \mb R^{N \times d}$ and $k$. Let $\mc X : = \mc I \cup \mc N$. Let $\mc I := \cup_{i=1}^k B_i$ where $B_i$ is a ball of radius at most one centered at $\mu_i$ and $\|\mu_i - \mu_j\| > \delta$. Let $\mc I' = \cup_i B_i'$ where $B_i' := \{x - \mu_i : x \in B_i\}$.  Let $\mc N = \mc N_1 \cup \mc N_2$ have the following properties. For all $p \in \mc N_1$ and for all $i, j$, we have that $| \|p- \mu_i\| - \|p- \mu_j\|| \ge \alpha$. For all $p \in \mc N_2$ and for all $x \in I, \|p - x\| \ge 2\delta$. Note that $\mc N_1 \cap \mc N_2 = \phi$. Let $n = \min_i |B_i|$ and $\epsilon = \frac{|\mc N_1|}{n}$. If $\frac{|\mc N_2|}{n} \le \frac{\delta^2(1-4\epsilon^2)-2\delta(1+4\epsilon)}{\lambda}$ and 
\begin{enumerate}[leftmargin=*,nolistsep,noitemsep]
	\item \textbf{Distribution free} - If  the distance between the centers of any two balls
	\begin{align*}
		\delta > 2 + \frac{9\epsilon}{1-4\epsilon^2} + \sqrt{\frac{2\sigma_{\max}^2(I')}{n(1-4\epsilon^2)}} \text{ and }
		\alpha \ge \sqrt{10\delta \epsilon + 4\delta^2 \epsilon^2+ \frac{2\sigma_{\max}^2(|I'|)}{n(1-4\epsilon^2)}}
	\end{align*} 
	then the regularised $k$-means SDP finds the intended cluster solution  $\mc C^* = \{C_1, \ldots, C_k, \mc N_2\}$ where $B_i \subseteq C_i$ when given $\mc X$ and $\delta^2+2\delta \ge \lambda \ge (\delta-1)^2 + 1$ as input.
	\item  \textbf{Stochastic ball assumption} - Let $\mc P$ denote the isotropic distribution on the unit ball centered at origin. Given centers $c_1, \ldots, c_k$ such that $\|c_i - c_j\| > \delta > 2$. Let $\mc P_i$ be the measure $\mc P$ translated with respect to the center $c_i$. If each $B_i$ is drawn i.i.d w.r.t the distribution $\mc P_i$ 
	\begin{align*}
	   \delta > 2 + \frac{9\epsilon}{1-4\epsilon^2} + \sqrt{\frac{2\rho k\theta(1+\frac{1}{\log|I|})^2}{d(1-4\epsilon^2)}} \text{ and }
	  \alpha \ge \sqrt{10\delta \epsilon + 4\delta^2 \epsilon^2+ \frac{2\rho k\theta(1+\frac{1}{\log|I|})^2}{d(1-4\epsilon^2)}}
	\end{align*}
	then there exists a constant $c_2 > 0$ such that with probability at least $1 - 2d\exp(\frac{-c_2|I|\theta}{d\log^2|I|})$ the regularised $k$-means SDP finds the intended cluster solution  $\mc C^* = \{C_1, \ldots, C_k, \mc N_2\}$ where $B_i \subseteq C_i$ when given $\mc X$ and $\delta^2+2\delta \ge \lambda \ge (\delta-1)^2 + 1$ as input.
\end{enumerate}
\end{theorem}

The proof of both the Thms. \ref{thm:SDPGeneral} and \ref{thm:regularisedSDPGeneral} use the following ideas. We construct a dual for the SDP. We then show that when the conditions of our theorems are satisfied then there exists a feasible solution for the dual program. Moreover, the objective function value of primal and dual sdp program are the same. Hence, the solution found is indeed optimal. The same idea was also used in the proof of Thm. 11 in \cite{awasthi2015relax}. However, our analysis is tighter which helps us to obtain better bounds. The details are in the supplementary section.

We have also developed a regularised version of the $k$-means LP based algorithm. However, due to space constraints we could not include it in the main version of the paper. However, the details and the robustness guarentees for the LP based algorithm are in the appendix. 

\section{Experiments}
We ran several experiments to analyse the performance of our regularised $k$-means algorithm. The first set of experiements were simulations done on synthetic data. The second set of experiments were done on real world datasets like MNIST where we compared the performance of our algorithm against other popular clustering algorithms like $k$-means++. All our experiments were run on Matlab. We solved the SDP formulation using the Matlab SDPNAL+ package \cite{yang2015sdpnal+}. To run $k$-means++ we used the standard implementation of the algorithm available on Matlab.

\subsection{Simulation studies}
The goal of these sets of experiments was to understand the effect of different parameters on the performance of the regularised SDP algorithm. Given the number of clusters $k$, the separation between the clusters $\delta$, the dimension of the space $d$, the number of points in each cluster $n$ and the number of noisy points $m$. We generate a clustering instance $\mc X$ in $\mb R^d$ as follows. We first pick $k$ seed points $\mu_1, \ldots, \mu_k$ such that each of these points are separated by atleast $\delta$. Next we generate $n$ points in the unit ball centered at each of the $\mu_i's$. Finally we add $m$ points uniformly at random. 

We analyse the performance of the regularised SDP algorithm as the parameters change. The most crucial amongst them is the separation between the clusters $\delta$, the regularization constant $\lambda$ and the dimension $d$. Fig. \ref{figure:simulation} shows the heatmap under different parameter values. For each setting of the parameters, we generated 50 random clustering instances. We then calculated the fraction of times the regularised sdp was able to recover the true clustering of the data. If the fraction is close to one, then its color on the plot is light. Darker colors represent values close to zero.

We see an interesting transition for $\lambda$. When $\lambda$ is `too small' then the probability of recovering the true clustering is also low. As $\lambda$ increases the probability of success goes up which are represented by the light colors. However, if we increase $\lambda$ to a very high value then the success probability again goes down. This shows that there is a `right' range of $\lambda$ as was also predicted by our theoretical analysis. 

Another parameter of interest is the dimension of the space $d$. Note that from our theoretical analysis, we know that both the probability of success and the separation depend on $d$. Fig. \ref{figure:simulation} shows that for very low dimension, the regularised sdp fails to perfectly recover the underlying clustering. However, as the dimension grows so does the probability of success. For these two simulations, we fixed the number of points per cluster $n = 30$, $k = 8$ and the number of noisy points $m = 30$. We have similar plots for $(\delta, n)$ and $(\delta, k)$ and $(\delta, m)$ and $(n, m)$. These plots very mostly light colored as long as the number of noisy points was not too large ($\frac{m}{n} \le 5$). Hence, due to space constraints, we have included them only in the supplementary section.
   
\begin{figure}[t]
  \label{figure:simulation}
  \centering
  \includegraphics[width=0.4\textwidth]{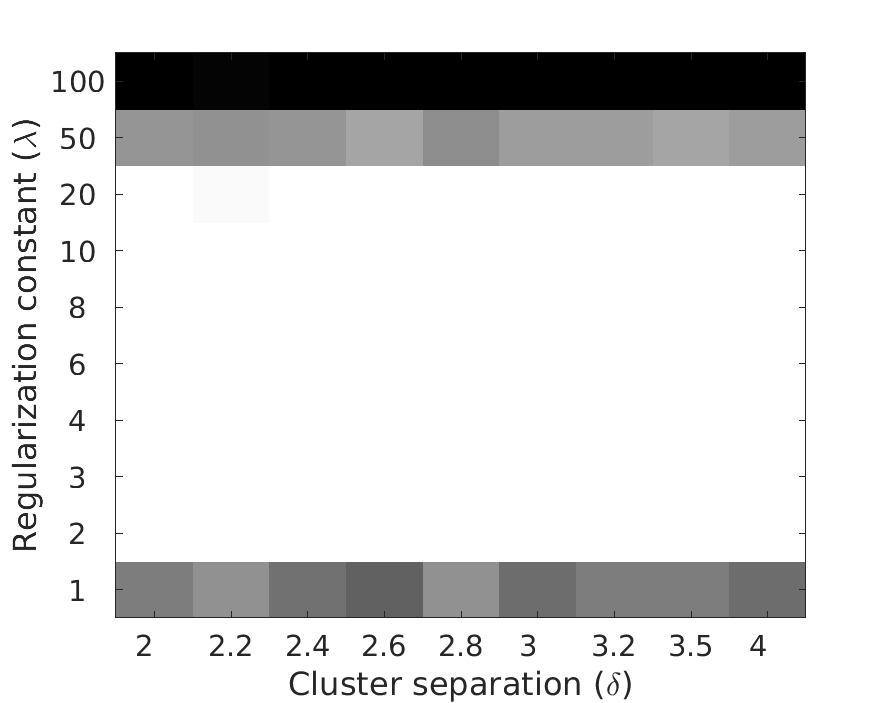}
  \includegraphics[width=0.4\textwidth]{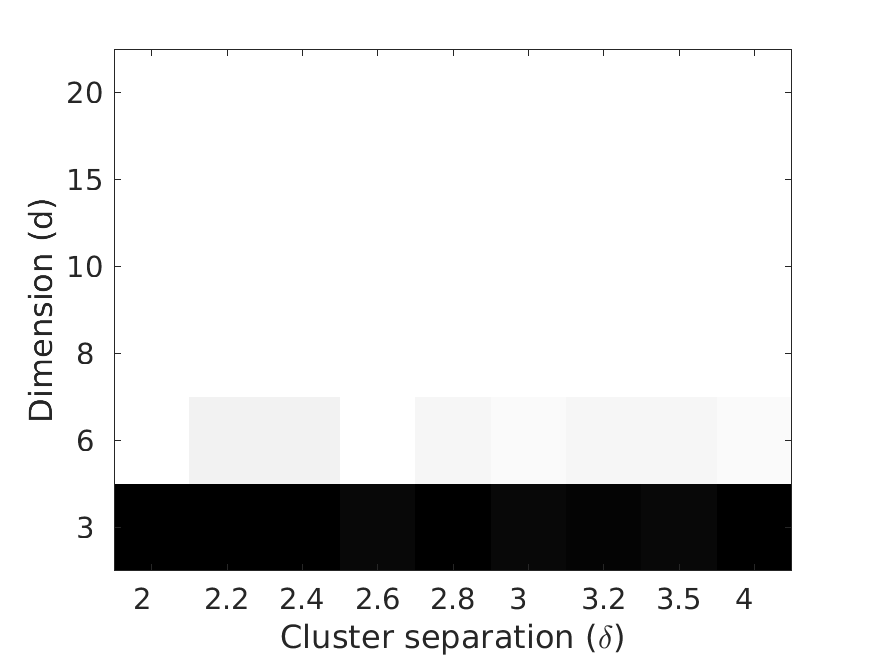}
  \caption{Heatmap showing the probability of success of the $k$-means regularised sdp algorithm. Lighter color indicates probability closer to one while darker indicates probability closer to zero.}  
\end{figure}

\subsection{Results on MNIST dataset}
We compare our regularised SDP algorithm against $k$-means++ on the MNIST dataset. MNIST is a dataset of images of handwritten digits from zero to nine. It contains 60,000 training images and 10,000 test images. We choose $k = 4$ different classes and randomly sample a total of $N = 1,000$ images from these classes.  We then run both our regularised SDP algorithm and the $k$-means++ algorithm on this dataset. We repeat this process for 10 different random samples of MNIST. We measure the performance of the two algorithms in terms of the precision and recall over the pairs of points in the same cluster. Given a clustering $\mc C$ and some target clustering $C^*$. Define the precision $p$ of $\mc C$ as the fraction of pairs that were in the same cluster according to $\mc C^*$ given that they were in the same cluster according to $\mc C$. The recall $r$ of $\mc C$ is the fraction of pairs that were in the same clustering according to $\mc C$ given that they were in the same cluster according to $\mc C^*$. We finally measure the $f_1$ score of the clustering $\mc C$ as the harmonic mean of its precision and recall. $f_1 = \frac{2pr}{p+r}$. 

Note that the regularised algorithm outputs $k+1$ clusters. Hence, to make a fair comparison, we finally assign each point in the noisy cluster ($C_{k+1}$) to one of the clusters $C_1, \ldots, C_k$ depending upon the distance of the point to the clusters. Another point is that the $f_1$ measures are sensitive to the choice of the $k$ digits or classes. For some choice of $k$ classes, the $f_1$ measures for both the algorithms are higher than compared to other classes. This shows that some classes are more difficult to cluster than other classes. Hence, we only report the difference in performance of the two algorithms. 

We report the performance on datasets with and without noisy points. The first is when there are no outliers or noisy points. In this case, the difference in the $f_1$ values was about $4.34\%$ in favor of $k$-means++. We then added noisy points to the dataset. In the first case, we added images from different datasets like EMNIST (images of handwritten letters). In this case, the difference was $2.54\%$ in favor of the regularised algorithm. In the second case, besides images from different datasets, we also added a few random noisy points to the MNIST dataset. In this case, the difference increased further to about $6.9\%$ in favor of the regularised algorithm. 

\section{Conclusion}
We introduced a regularisation paradigm which can transform any center-based clustering objective to one that is more robust to the addition of noisy points.  We proved that regularised objective is NP-Hard for common cost functions like $k$-means. We then obtained regularised versions of an existing clustering algorithm based on convex (sdp) relaxation of the $k$-means cost. We then proved noise robustness guarentees for the regularised algorithm. The proof improved existing bounds (in terms of cluster separation) for sdp-based standard (non-regularised) $k$-means algorithm. Our experiments showed that regularised sdp-based $k$-means performed better than existing algorithms like $k$-means++ on MNIST especially in the presence of noisy points.
 
\subsection*{Acknowledgements}
We would like to thank Nicole McNabb for helpful discussions on the topics in this paper.
\bibliographystyle{alpha}
\bibliography{optimizationClustering}

\appendix
\section{Proof of theorems and lemmas}
\subsection{Hardness of regularised $k$-means}
\label{a-section:hardness}
\begin{theorem}
\label{a-theorem:hardFork1Fixed}
Given a clustering instance $\mc X \subset \mb R^{d}$, define $m(\mc X) := \min_{x \neq y \in \mc X} \|x-y\|_{2}^2$ and $n := |\mc X|$. Finding the optimal solution to the regularised $1$-means objective is NP-Hard for $\frac{m(\mc X)}{2} < \lambda < \frac{m(\mc X)}{2} + \frac{1}{2n^2(n-1)}$. 

In particular, the optimization problem is NP-Hard for $\lambda = \lambda_0(\mc X) := \frac{m(\mc X)}{2} + \frac{1}{4n^3}$.
\end{theorem}

\begin{proof}
The proof uses reduction from the clique problem. Given a graph $G = (V, E)$ and an integer $q$, the clique problem asks the following: does there exist a clique in $G$ of size at least $q$?

Given an instance of the clique problem, we construct an instance of regularised $1$-means as follows. For every $v \in V$, construct $x_v \in \mc X$. Define the metric as 
\[
d^2(x_i, x_j) = 
\begin{cases}
1 \hspace{0.51in}\text{if }(i, j) \in E\\
1 + \Delta \hspace{0.221in}\text{if }(i, j) \not\in E\numberthis\label{a-eqn:dmetric}
\end{cases}
\]
where $0 < n\Delta < 1 $. Now, we will show that $G$ has a clique of size $ \ge q$ $\iff$$\mc X$ has a clustering of cost $\le c = \frac{q-1}{4} + \lambda (n - q)$.

$\Longrightarrow$ Assume $G$ has a clique of size at least $q$. Assign all points in the clique to $C_1$ and the remaining points to $C_2$. This clustering has cost as desired.

$\Longleftarrow$ Let $|C_1| = n'$. Now, there are three possibilities.

Case 1: $n' > q$. If all the distances in $C_1$ are 1, then the cost of the clustering is $\frac{n'-1}{4} + \lambda (n - n') \le c$ as $\lambda > \frac{1}{2}$. Thus, the vertices corresponding to the points in $C_1$ form a clique of size $n' > q$. If at least one distance is $1 + \Delta$ then the cost of the clustering is 
\begin{flalign*}
\frac{n'-1}{4} + \lambda (n - n') + \frac{\Delta}{n'} \le c \implies \lambda \ge \frac{1}{2} + \frac{\Delta}{n'(n'-q)}
\end{flalign*}
This is a contradiction because of the choice of $\lambda$.

Case 2: $n' < q$. If all the distances in $C_1 = 1$, then the cost of the clustering is $\frac{n'-1}{2} + \lambda (n - n') > c$. If atleast one distance is $1 + \Delta$ then the cost of the clustering is even greater as $\lambda > \frac{1}{2}$. 

Case 3: $n' = q$. If atleast one distance is $1 + \Delta$ then the cost of the clustering is $\frac{q-1}{2} + \lambda (n - q) + \frac{\Delta}{m} > c$ . 
Hence, the only possibility remains that $|C_1| = q$ and all the distances in $C_1 = 1$. Hence, $G$ has a clique of size $q$.
\end{proof}

\begin{lemma}
Let $d$ be as in Eqn. \ref{a-eqn:dmetric}. Then $d$ can be embedded into $\mb R^{|\mc X|}$.
\end{lemma}
\begin{proof}
The proof of embedding is very similar to Thm. 8 in \cite{dasgupta2008hardness}. Using Cor. 7 in \cite{dasgupta2008hardness}, we know that $D$ can be embedded into $\mb R^{|\mc X|}$ if and only if $u^TDu \le 0$ for all $u^T1 = 0$.
\begin{align*}
u^TDu &= \sum_{ij} u_{i}d_{ij}u_{j} = \sum_{ij}u_iu_j \Big(1 - \mb 1(i=j) + \Delta \mb 1(i \not\sim_{E} j) \big)\\
&\le \Big(\sum_{i}u_i \Big)^2 - \sum_i u_i^2 +\Delta\sum_{ij}|u_i||u_j| \le -\|u\|^2 + |\mc X|\Delta \|u\|^2 = -(1-|\mc X|\Delta)\|u\|^2,
\end{align*}
which completes our proof.
\end{proof}

We now use Thm. \ref{a-theorem:hardFork1Fixed} to show that regularised $1$-means is Np-Hard for any $\lambda > m(\mc X) / 2$.

\begin{restatable}{theorem}{hardForkone}
\label{a-theorem:hardFork1}
Given a clustering instance $\mc X \subset \mb R^d$. Finding the optimal solution to the regularised $1$-means objective is NP-Hard for all $\lambda > m(\mc X) / 2$.
\end{restatable}
\begin{proof}
Denote $d(x,y)$ to be the $l_2$ distance. Thm. \ref{a-theorem:hardFork1Fixed} showed that optimizing the problem  
\begin{align*}
\min_{C\subseteq \mc X} \enspace \sum_{x \in C}d^2(x, c) + \lambda_0(\mc X) \enspace |\mc X\setminus \mc C| \hspace{0.5in}\text{(FRM)}
\end{align*}
is NP-Hard when $\lambda_0(\mc X) := \frac{m(\mc X)}{2} +\frac{1}{4n^3}$. Given $\lambda>0$, we want to show that for all $\lambda > m(\mc X) / 2$, the following is also NP-Hard to optimize: 
\begin{align*}
\min_{C\subseteq \mc X} \enspace \sum_{x \in C}d^2(x, c) + \lambda \enspace |\mc X\setminus \mc C| \hspace{0.5in}\text{(GRM)}
\end{align*}

Let $\lambda' = \sqrt{n^3(4\lambda - 2m(\mc X))} > 0$, and define a new Euclidean distance $d'(x, y) = \frac{d(x,y)}{\lambda'}$. Then
\begin{flalign*}
\sum_{x \in C}d^2(x, c) + \lambda \enspace |\mc X\setminus \mc C| &= \lambda'^2\Bigg(\sum_C d'^2(x, c) + \frac{|\mc X \setminus \mc C|}{\lambda'^2} \Big(\frac{m(\mc X)}{2} +\frac{\lambda'^2}{4n^3}\Big)\Bigg)\\
&= \lambda'^2\Big(\sum_C d'^2(x, c) + |\mc X \setminus \mc C| \lambda_0(\mc X')\Big).&
\end{flalign*}
Hence, we see that GRM is equivalent to FRM which is NP-Hard.
\end{proof}

\subsection{The regularised $k$-means algorithm}
\label{a-section:heuristic}
 
\subsubsection{Equivalence of regularised $k$-means with $0,1$-SDP}
\label{a-subsection:modifiedkmeans01sdp}
We follow a similar technique to that of \cite{peng2007approximating} to translate equation \ref{eqn:modifiedkmeans} into a 0-1 SDP. We are given a set $\mc X$ with $n$ data points. The goal is to partition the set into $k$ clusters with the option of throwing some points into the garbage $k+1$ cluster. Let $S$ be an assignment matrix of size $n\times k$ and $y$ be an $n \times 1$ column vector that assigns points to the ``garbage'' cluster.  

 \[s_{ij} = 
    \begin{dcases}
		\hspace{0.1in} 1 \enspace \text{iff }x_i \in C_j, j \leq k\\
		\hspace{0.1in} 0 \enspace \text{otherwise}
	\end{dcases}
 \quad y_{i} = 
    \begin{dcases}
		\hspace{0.1in} 1 \enspace \text{iff }x_i \in C_{k+1}\\
		\hspace{0.1in} 0 \enspace \text{otherwise}
	\end{dcases}
\]

Provided that $C_i\ne\emptyset$, we have the following equality by expanding the mean $c_i = \frac{1}{|C_i|} \sum_{x\in C_i} x$:
\begin{align}
\label{a-eqn:kmeansobj1}
\sum_{x \in C_i} \|x-c_i\|^2 
&= \frac{1}{2|C_i|}\sum_{x, y \in C_i} \|x-y\|^2 
\end{align}

\noindent $c_j = \frac{\sum_{i} s_{ij} x_i}{\sum_{i} s_{ij}}$ is the average of points in the $j^{\text{th}}$ cluster. Hence,
\begin{align*}
	&\sum_j \sum_i s_{ij}(\langle c_j, c_j\rangle -2\langle x_i, c_j\rangle) = 	 \sum_j \langle\sum_i s_{ij} c_j, \enspace c_j\rangle - 2 \sum_j \langle\sum_i s_{ij} x_i, \enspace c_j\rangle \\
	&=\sum_j \langle c_j, \sum_i s_{ij}x_i\rangle - 2 \sum_j \langle\sum_i s_{ij} x_i, \enspace c_j\rangle = -\sum_j \langle\frac{\sum_i s_{ij} x_i}{\sum_i s_{ij}}, \sum_i s_{ij}x_i\rangle = -\sum_j \frac{\| \sum_i s_{ij} x_i\|^2}{\sum_i s_{ij}} \numberthis \label{a-eqn:kmeansobj2}
\end{align*}

\noindent Combining equations \ref{a-eqn:kmeansobj1} and \ref{a-eqn:kmeansobj2}, we get that 
\begin{align*}
  &\sum_{j=1}^k \sum_{x \in C_i} \|x-c_i\|^2 = \sum_{i=1}^n \|x_i\|^2 (1-y_i) - \sum_{j=1}^k \frac{\| \sum_i s_{ij} x_i\|^2}{\sum_i s_{ij}} = \sum_{i=1}^n \|x_i\|^2 (1-y_i) - \sum_{i=1}^n\sum_{j=1}^n \langle x_i, x_j\rangle Z_{ij} \label{a-eqn:halfcost}\numberthis
\end{align*}

where $Z = S (S^T S)^{-1} S^T$. Observe that if $x_i \not\in C_{k+1}$ then $Z_{ij} = \frac{1}{|S_{c(i)}|} \langle s_i, s_j\rangle$ where $S_{c(i)}$ denotes the size of the $i^{th}$ cluster. If $x_i \in C_{k+1}$ then $Z_{ij}=0$. Thus, 
\[
	Z_{ij} = 
	\begin{dcases}
	\frac{1}{|S_{c(i)}|} \langle s_i, s_j\rangle \quad &\text{if }y_i = 0\\
	0 \quad &\text{otherwise }
	\end{dcases}
\]
Observe that $Z_{ij} = Z_{ji}$ and
\begin{align*}
\tr(Z) &= \sum_i Z_{ii} = \sum_{x \not\in C_{k+1}} \frac{1}{|S_{c(i)}|} = k. 
\end{align*}
Also, we have that  and $\langle Z_i, \mb 1\rangle = \sum_{j} Z_{ij}$. If $y_i = 0$ then $\sum_{j}Z_{ij} = 0$ else $\sum_j Z_{ij} = \frac{1}{|S_{c(i)}|}\sum_j \langle s_i, s_j\rangle = 1$. Hence, we get that $\langle Z_i, \mb 1\rangle = \sum_j Z_{ij} = 1-y_i$, or equivalently, $Z\cdot \mb 1 + y = \mb 1$. Also, it is fairly easy to see that $Z^2 = Z$.

Let $D$ be a matrix such that $D_{ij} = d^2(x_i, x_j)$. Using the above properties of $Z$, we get that
\begin{align*}
&\tr(DZ) = \sum_{ij} \langle x_i - x_j, x_i-x_j\rangle z_{ij} = \sum_i\|x_i\|^2\sum_j z_{ij} + \sum_j\|x_j\|^2\sum_i z_{ij} -2\sum_{ij}\langle x_i, x_j\rangle z_{ij} \\
&= 2\Big( \sum_i\|x_i\|^2 (1-y_i) - \sum_{ij}\langle x_i, x_j\rangle z_{ij}\Big) = 2 \sum_{i=1}^k \sum_{x \in C_i} \|x-c_i\|^2 \text{ (using Eqn. \ref{a-eqn:halfcost})}
\end{align*}

Finally, observe that $\lambda |C_{k+1}| = \lambda\langle \mb 1, y\rangle$.

\begin{equation*}
	\begin{split}
	\textbf{0-1}\\
	\textbf{SDP}
  \end{split}
	\begin{cases}
		\min_{Z, y} \enspace &\tr(DZ) + \lambda \langle \mb 1, y\rangle\\
		\text{s.t. } \enspace &\tr(Z) = k\\
		& Z\cdot \mb 1 + y = \mb 1\\	
		&Z\ge 0, Z^2 = Z, Z^T = Z \\
		& y \in \{0, 1\}^n
	\end{cases}
	\xrightarrow{\text{relaxed}} \textbf{ SDP } 
	\begin{cases}
		\min_{Z, y} \enspace &\tr(DZ) + \lambda \langle \mb 1, y\rangle\\
        \text{s.t. } \enspace &\tr(Z) = k\\
		& \Big(\frac{Z+Z^T}{2}\Big)\cdot \mb 1 + y = \mb 1\\		
		&Z \ge 0, y \ge 0, Z \succeq 0 \numberthis\label{a-eqn:regularisedSDP}
	\end{cases}
\end{equation*}

\begin{restatable}{theorem}{modifiedkmeans}
\label{a-thm:modifiedkmeans}
Finding a solution to the 0-1 SDP (\ref{eqn:regularisedSDP}) is equivalent to finding a solution to the regularised $k$-means objective (\ref{eqn:modifiedkmeans}). 
\end{restatable}

\begin{proof}
We will use the same proof ideas as in the proof of Theorem 2.2 in \cite{peng2007approximating}. However, we need to modify the proof slightly according to our formulation. From the discussion in the previous subsection, we can see that any solution for (\ref{eqn:modifiedkmeans}) implies a solution for the 0-1 SDP (\ref{eqn:regularisedSDP}) with same cost. Now, we will prove the other direction. Any solution for the 0-1 SDP implies a solution for for (\ref{eqn:modifiedkmeans}) with same cost.

Let $e_i$ be a vector with all zeros except in the $i^{\text{th}}$ index. Observe that $u_i^T Z u_i \ge 0$. Hence, $Z_{ii} \ge 0$. Similarly, $(e_i-e_j)^T Z (e_i-e_j) = Z_{ii} - 2Z_{ij} + Z_{jj}$. Hence, $Z_{ij} \le \max (Z_{ii}, Z_{jj})$. Let $Z_{i^*i^*} = \max_i Z_{ii}$. Hence, we have that for all $i, j$, $Z_{ij} \le Z_{i^* i^*}$. 

Suppose $Z_{i^*i^*} = 0$. Then, $Z = \textbf{0}$ and $y = \textbf{1}$. This implies that all points are assigned to the $k+1$ cluster. The cost of both the solutions in this case is $\lambda |C_{k+1}|$.

Now, suppose $Z_{i^*i^*}>0$. Let $I =\{j: Z_{i^* j} > 0\}$. Since, $Z^2 = Z$, we have that $Z_{i^*i^*} = \sum_{j=1}^n Z^2_{i^* j} = \sum_{j \in I} Z^2_{i^*j}$. Hence, $\sum_{j \in I}\frac{Z_{i^*j}}{Z_{i^*i^*}}Z_{i^*, j} = 1$. Also, we have that $\sum_{j} Z_{i^*j} + y_{i^*} = 1$. If $y_{i^*} = 1$, then $\sum_{j} Z_{i^*j} = 0$ which contradicts our assumption that $Z_{i^*i^*} > 0$. Hence, $y_{i^*} = 0$ and we have $\sum_{j} Z_{i^*j} = \sum_{j\in I} Z_{i^*j} = 1$. Since we have the following constraints,
\begin{align*}
	\sum_{j\in I} Z_{i^*j} = 1 \hspace{0.4in}\text{and} \hspace{0.4in}\sum_{j\in I} \frac{Z_{i^*j}}{Z_{i^*i^*}}Z_{i^*j} = 1
\end{align*}
$Z_{i^*j} = Z_{i^*i^*}$ for all $j \in I$. Hence, we see that the matrix $Z$ and the vector $y$ can be decomposed as 
\[ Z = 
\begin{bmatrix}
    Z_{II}  & 0 \\
    0       & Z'
\end{bmatrix}
\quad y=\begin{bmatrix}
    0 \\
    y' 
\end{bmatrix}
\]
where $Z_{II} = \frac{1}{|I|}\mb1_{|I|}\mb1_{|I|}^T$. Now, we can see that $\tr(Z') = k-1$ and
$$Z\mb 1 + a = \begin{bmatrix}1\\Z'\mb 1\end{bmatrix} + \begin{bmatrix}0\\a'\end{bmatrix} = \begin{bmatrix}1\\1\end{bmatrix}.$$

This implies that $Z'\cdot \mb 1 + y' = \textbf{1}$. Hence, the optimization problem now reduces to
 \[
    \begin{cases}
		\min_{Z',y'} \enspace &\tr(DZ') + \lambda\langle \cdot \mb 1, y'\rangle\\
		\text{subject to } \enspace &\tr(Z') = k-1\\
		&Z'\cdot \mb 1 + y' = \mb 1
	\end{cases}
\]
Repeating this process $k$ times, we get that $Z$ can be decomposed into $k$ non-zero block diagonal matrices and one zero block diagonal matrix. Hence, using this we construct a solution for the original clustering problem as follows. For all $i$, if the row $Z_i$ belongs to the $j^{th}$ diagonal block then $x_i$ is assigned to $C_i$. Given $Z$ and $a$, the cost of the 0-1 SDP solution is
\begin{align*}
\frac{1}{2} \tr(DZ) + \lambda\langle\mb 1, y\rangle &= \sum_{i=1}^k\sum_{x,y \in C_i} \frac{\|x_i-x_j\|^2}{2|C_i|} + \lambda|C_{k+1}| = \sum_{i=1}^k \sum_{x \in C_i} \|x-c_i\|^2 + \lambda |C_{k+1}|
\end{align*}
which is the same as the cost of the regularised $k$-means objective. Hence, from a feasible solution of (\ref{eqn:modifiedkmeans}), we can obtain a feasible solution of the 0-1 SDP of same cost.
\end{proof}

\subsection{Tightness of the SDP based algorithm}
\begin{equation*}
	\textbf{0-1 SDP} 
	\begin{cases}
		\min_{Z} \enspace &\tr(DZ) \\
		\text{subject to } \enspace &\tr(Z) = k\\
		& Z\cdot \mb 1 = \mb 1\\	
		&Z\ge 0, Z^2 = Z, Z^T = Z 
	\end{cases}
	\xrightarrow{\text{relaxed}} \textbf{ SDP } 
	\begin{cases}
		\min_{Z} \enspace &\tr(DZ)\\
        \text{subject to } \enspace &\tr(Z) = k\\
		& \Big(\frac{Z+Z^T}{2}\Big)\cdot \mb 1 = \mb 1\\		
		&Z \ge 0, Z \succeq 0 \numberthis\label{a-eqn:SDP}
	\end{cases}
\end{equation*}

We are given a set $\mc X \subset \mb R^d$. $\mc X$ can be covered by a set of $k$ ``well-separated'' balls. That is, $\mc X := \cup_{i=1}^k B_i$ where $B_i$ is a ball of radius at most $r$ centered at $\mu_i$ and $\|\mu_i - \mu_j\| \ge \delta r$. 
Define $n_i := |B_i|$ and $n := \min_{i\in[k]} n_i$ and $N = \sum_i n_i$. $D$ is an $N\times N$ matrix such that $D_{ij} = \|x_i -x_j\|^2$.  

The goal is to output a clustering $\mc C^*$ of $\mc X$ such that $C^* = \{B_1, B_2, \ldots, B_k\}$. From the way we constructed the 0-1 SDP, this corresponds to
\begin{align}
  Z^* = \sum_{p=1}^k \frac{\mb 1^p \mb 1^{p^T}}{n_p} \label{a-eqn:intendedsolution}
\end{align}
where $\mb 1^p$ is an $N$-dimensional indicator vector for the $p^{th}$ cluster. That is, $Z$ is a block diagonal matrix and consists of $k$ non-zero diagonal blocks. Observe that $Z$ consists of blocks $Z^{(p, q)}$. Also, $Z^{(p, p)} = \frac{1}{n_p}\mb 1\mb 1^T$ and for $p \neq q, Z^{(p, q)} = 0$. To prove that our SDP (Eqn. \ref{a-eqn:SDP}) finds this solution, we  will adopt the following strategy. We first construct a dual for Eqn. \ref{a-eqn:SDP}. We then show that under certain conditions on $\delta$ (well-separateness of the balls) the following happens. The primal objective value and the dual objective value are the same. Also, the corresponding $Z$ satisfies Eqn. \ref{a-eqn:intendedsolution}. 

Before, we describe the dual, lets introduce a bit of notation. We index every point as $(p, i)$ where $p$ denotes the ball (or cluster) to which it belongs and $i$ denotes the index within that ball. Observe that the distance matrix $D$ consists of blocks $D^{(p, q)}$ such that $D^{(p,q)}_{ij} = \|x_{(p,i)}-x_{(q,j)}\|^2_2$. Now, to construct the dual, we introduce variables $z, \alpha_{(p,i)}, \beta_{(p,i)(q,j)}$ and $Q$ for each of the constraints in the primal problem. 
\[\textbf{SDP Dual}
    \begin{cases}
		\max \enspace &-zk - \sum_{p=1}^{k}\sum_{i=1}^{n_p} \alpha_{(p,i)}\\
		\text{subject to } \\
		&Q = D + zI + \sum_p\sum_i \alpha_{(p,i)}A_{(p, i)} -\sum_{p,q}\sum_{i,j} \beta_{(p,i)(q,j)}E_{(p,i)(q,j)}\\
		&\beta \ge 0 \\
		&Q \succeq 0
	\end{cases}
	\label{a-eqn:SDPDual}
	\numberthis
\]
where, $A_{(p,i)} = \frac{1}{2}(e_{p,i}\mb 1^T + \mb 1 e_{p, i}^T)$ and $E_{(p,i)(q,j)} = e_{(p,i)}e_{(q,j)}^T$. $\mb 1$ is an $N$-dimensional vector of all ones while $e_{p,i}$ is the indicator vector with one in position $(p, i)$ and zeros elsewhere. Now, we will examine the conditions under which the dual objective value matches the primal objective such that all the constraints of the dual are satisfied. 
\subsubsection*{Complementary slackness}
We know that $\beta_{(p, i)(q, j)} Z_{(p, i)(q, j)} = 0$. Now, $Z^{(p, q)} = 0$ and $Z^{(p, p) \neq 0}$. Hence, we get that $\beta^{(p, p)} = 0$. Also, if we have that $Q^{(p,p)}\mb1 = 0$ then $\langle Q, Z \rangle = 0$. This ensures that the second complementary slackness condition is also satisfied. 

\subsubsection*{Some properties of the $Q$ matrix}
Before we proceed, let's examine some properties of the dual matrix $Q$. Observe that for all $1 \le p\neq q \le k$, 
\begin{alignat*}{1}
&Q^{(p,p)} = D^{(p,p)} + zI_{n_p} + \frac{1}{2}\sum_{i=1}^{n_p} \alpha_{(p,i)} (e_{i}\mb 1^T + \mb 1e_{i}^T)\\
&Q^{(p,q)} = D^{(p,q)} + \frac{1}{2}\sum_{i=1}^{n_p} \alpha_{(p,i)} e_{i}\mb 1^T + \frac{1}{2}\sum_{i=1}^{n_q}\alpha_{(q,i)}\mb 1e_{i}^T - \beta^{(p,q)}
\label{a-eqn:q}\numberthis
\end{alignat*} 

\subsubsection*{Dual matches intended primal solution}
Now, using the fact that $Q^{(p, p)}\mb 1 = 0$ implies that
\begin{align*}
  0 &= e_r^TD^{(p,p)}\mb1 + z + \frac{1}{2}\sum\alpha_{(p,i)} (e_r^Te_{i}\mb 1^T\mb1 + \mb e_r^T\mb1e_{i}^T\mb1) = e_r^TD^{(p,p)}\mb1 + z\\
  & + \frac{1}{2}\sum\alpha_{(p,i)} (e_r^Te_{i}n_p + 1) = e_r^TD^{(p,p)}\mb1 + z + \frac{1}{2}\sum\alpha_{(p,i)} +\frac{n_p\alpha_{(p,r)}}{2}.
\end{align*}
Summing over all $r$ and then substituiting back, we get that for all $1\le p \le k$ and for all $i$,  
\begin{align*}
  &\alpha_{(p, i)} = \frac{1^TD^{(p,p)}1}{n_p^2}-\frac{z}{n_p} -\frac{2e_i^TD^{(p,p)}1}{n_p}\\
  & = \frac{\sum_{i,j}\langle x_{pi}-x_{pj}, x_{pi}-x_{pj}\rangle -2n_p \sum_j\langle x_{pi}-x_{pj}, x_{pi}-x_{pj}\rangle}{n_p^2}-\frac{z}{n_p}\\
  &= \frac{-2n_p^2 \|x_{pi}\|^2 + 4n_p \langle x_{pi}, \sum_jx_{pj}\rangle - 2\langle \sum_i x_{pi}, \sum_j x_{pj}\rangle -zn_p}{n_p^2}\\
  &= -2\|x_{pi} - x_p\|^2-\frac{z}{n_p} \hspace{1in}\text{ ($x_p$ denotes the center of the $p^th$ cluster)}\numberthis \label{a-eqn:alphapi}
\end{align*}
Now, we have the value of $\alpha$ for all $1\le p \le k$ and for all $i$. Computing the objective function, we get that 
\begin{align*}
  kz + \sum_p\sum_i \alpha_{p,i} &= kz + \sum_{p}\frac{1^TD^{(p, p)}1}{n_p} - \sum_p z -2\sum_p \frac{1^TD^{(p,p)}1}{n_p} = -\sum_p \frac{\mb1^TD^{(a,a)}\mb1}{n_p}\\
  & = -\langle D, Z \rangle = -\tr(DZ)
\end{align*}
Hence, we see that for the intended solution, the primal and dual values are the same. Hence, solution is optimal. Now, the main question is to find $Q$ such that $Q$ is positive semi-definite while simultaneously ensuring that $\beta \ge 0$.  

\subsubsection*{Satisfying PSD for $Q$}
We already know that $Q^{(p,p)}\mb 1 = 0$. We will now try to ensure that $Q^{(p, q)}\mb 1 = 0$. As we will see later, this will help us to prove the positive semi-definiteness property for $Q$. Now, $Q^{(p,q)}1 = 0$ implies that for all $r$, we have $e_r^TQ^{(p,q)}1 = 0$.  
\begin{flalign*}
  &0 = e_r^TQ^{(p, q)}1 = \sum_s Q^{(p,q)}_{rs} = \sum_{s} D^{(p, q)}_{rs} + \frac{n_q \alpha_{(p,r)}}{2} + \frac{1}{2}\sum_{i=1}^{n_q}\alpha_{(q, i)} - \sum_{s}\beta^{(p, q)}_{rs}&
\end{flalign*}
It is always possible to satisfy the above equation by choosing $\beta_{rs}^{(p,q)}$ as long as it is greater than zero. That is we need that, 
\begin{align*}
  &\beta^{(p,q)}_{rs}  := \frac{\sum_{s} D^{(p, q)}_{rs}}{n_q} + \frac{\alpha_{(p,r)}}{2} + \frac{\sum_{i=1}^{n_q}\alpha_{(q, i)}}{2n_q} \ge 0\\
  &\iff \frac{\sum_{s} \|x_{pr}-x_{qs}\|^2}{n_q} \ge \frac{\sum_s\|x_{qs} - x_q\|^2}{n_q} + \|x_{pr} - x_p\|^2 + \frac{z}{2n_p} + \frac{z}{2n_q} \numberthis \label{a-eqn:qpq1tmp}
\end{align*}
Before, we go further lets examine,
\begin{align*}
  &\sum_{s} \|x_{pr}-x_{qs}\|^2 = n_q\|x_{pr}\|^2 - 2n_q\langle x_{pr}, x_q \rangle + \sum_s \|x_{qs}\|^2\\
  & = n_q \|x_{pr}-x_q\|^2 + \sum_s \langle x_{qs}, x_{qs}\rangle - n_q\langle x_q, x_q \rangle = n_q \|x_{pr}-x_q\|^2 + \sum_s \|x_{qs} - x_q \|^2
\end{align*} 
Substituting this in Eqn. \ref{a-eqn:qpq1tmp}, we get that it is always possible to satisfy $Q^{(p, q)}1 = 0$ as long as for all $r$, we have that 
\begin{align}
  \|x_{pr} - x_q\|^2 - \|x_{pr} - x_p\|^2 \ge \frac{z}{2n_p} + \frac{z}{2n_q}\label{a-eqn:zupper}
\end{align}
Also, note that from $\beta_{rs}^{(p,q)}$ as defined in Eqn. \ref{a-eqn:qpq1tmp}, we get that 
\begin{align}
  &Q_{rs}^{(p,q)} = D^{(p, q)}_{r,s} + \frac{1}{2}\alpha_{q, s} - \frac{1}{n_q}\sum_j D^{pq}_{rj} -\frac{1}{2}\frac{\sum_j \alpha_{qj}}{n_q} \hspace{0.3in}\text{ and from Eqn. \ref{a-eqn:q}}\\
  &Q_{rs}^{(p,p)} = D^{(p, p)}_{r,s} + \frac{1}{2}\alpha_{p, r} + \frac{1}{2}\alpha_{p, s} + z \mb 1_{[r=s]}
\end{align}
If Eqn.\ref{a-eqn:zupper} holds, then for all $1 \le p, q \le k$, we have that $Q^{(p,q)}\mb1 = 0$. Let $\mb 1_p$ denote the $N$-dimensional indicator vector for the $p^{th}$ cluster. Then, we see that for all $1 \le p \le k$, we have that $Q \mb 1_p = 0$. Let $V$ be the subspace spanned by these vectors. That is, $V = span\{1_p\}_{p=1}^k$. Then, for all $v \in V$, $v^T Q y = v^T \mb 0 = \mb 0$. Hence, we need to only show that for all $v \bot V$, $v^TQv \ge 0$. Let $v = [v_1, \ldots, v_k]^T$. Since, $v \bot V$, we know that for all $p$, $\langle v_p, 1\rangle = 0 = \sum_{r}v_{pr}$. Now,
\begin{align*}
  &v^TQv = \sum_{pq}\sum_{rs}x_{pr}Q^{(p, q)}_{rs}v_{qs} = \sum_{p \neq q}\sum_{rs}v_{pr}Q^{(p, q)}_{rs}v_{qs} + \sum_{p}\sum_{rs}v_{pr}Q^{(p, p)}_{rs}v_{qs}
\end{align*}
Now, we analyse the case when $p \neq q$. Then, we have that
\begin{align*}
  &\sum_{rs}v_{pr}Q^{pq}_{rs}v_{qs} = \sum_{rs}v_{pr}D^{pq}_{rs}v_{qs} + \frac{1}{2}\sum_{rs}\alpha_{qs}v_{qs}v_{pr} -\frac{1}{n_q}\sum_{rs} \sum_j v_{pr}v_{qs}D^{pq}_{rj} -\frac{1}{2n_q}\sum_{rs}\sum_jv_{pr}v_{qs}\alpha_{qj}\\
  &= \sum_{rs}v_{pr}v_{qs}D^{pq}_{rs} + \frac{1}{2}\sum_{s}\alpha_{qs}v_{qs}\sum_r v_{pr} -\frac{1}{n_q}\sum_{r}\sum_j D^{pq}_{rj}v_{pr}\sum_sv_{qs} -\frac{1}{2n_q}\sum_{s}\sum_jv_{qs}\alpha_{qj}\sum_{r}v_{pr}\\
  &= \sum_{rs}v_{pr}v_{qs}D^{pq}_{rs}
\end{align*}
Now for the other case, we have that
\begin{align*}
  \sum_{rs}v_{pr}Q^{pp}_{rs}v_{ps} &= \sum_{rs}v_{pr}D^{pp}_{rs}v_{ps} + \frac{1}{2}\sum_{rs}\alpha_{pr}v_{pr}v_{ps} + \frac{1}{2}\sum_{rs}\alpha_{ps}v_{pr}v_{ps} + \sum_r zv_{pr}v_{pr}\\
  &= \sum_{rs}v_{pr}D^{pp}_{rs}v_{ps} + \frac{1}{2}\sum_{r}\alpha_{pr}v_{pr}\sum_sv_{ps} + \frac{1}{2}\sum_{s}\alpha_{ps}v_{ps}\sum_r v_{pr} + \sum_r zv_{pr}v_{pr}\\
  &=\sum_{rs}v_{pr}D^{pp}_{rs}v_{ps} + \sum_r zv_{pr}v_{pr}
\end{align*}
Combining the above two equations, we get that 
\begin{align*}
  &v^TQv = \sum_{pq}v_{pr}D^{pq}v_{qs} + z\sum_p\sum_{r} (v_{pr})^2 = v^TDv + zv^Tv
\end{align*}
Now, let $X$ be the $N \times d$ dimensional input matrix. That is, the matrix $X$ contains the $N$ points in $d$ dimensional euclidean space. Then, $D = W + W^T - 2 XX^T$ where $W$ is a rank one matrix such that its $i^{th}$ contains $\|x_i\|^2$ in its $i^{th}$ row. That is, $W = \sum_i \|x_i\|^2 e_i\mb 1^T$. Now, $v \bot V$, hence we get that $v^TDv = -2v^TXX^Tv$. Thus, $Q$ is positive semi-definite as long as we can find $z$ such that  
\begin{align*}
  &z > 2 \max_{v\bot V} \frac{v^T XX^Tv}{v^Tv}. \label{a-eqn:zlower}\numberthis
\end{align*}

\subsubsection*{Putting it all together}
Eqns. \ref{a-eqn:zlower} and \ref{a-eqn:zupper}, show that as long as 
\begin{align}
  \|x_{pr} - x_q\|^2 - \|x_{pr}-x_p\|^2  > \frac{2}{n}\Big(\max_{v\bot V} \frac{v^T XX^Tv}{v^Tv}\Big)\label{a-eqn:mainConstraint}
\end{align}
  then we can find $Q$ and $\beta$ satisfying the constraints of the dual and there is no primal and dual gap. First observe that LHS of Eqn. \ref{a-eqn:mainConstraint} has a minimum of $(\delta - 1)^2 - 1$. Now, we need to upper bound the RHS of Eqn. \ref{a-eqn:mainConstraint}. Note that $X = X' + X$ where $C$ is a rank $k$ matrix which contains the centers $\mu_1, \ldots, \mu_k$. Also, for any $v \bot V$, we have that $v^TC = 0$. Let $\sigma_{\max}$ denote the maximum eigenvalue of the matrix $X$. Hence, 
\begin{align*}
  &\frac{2}{n}\Big(\max_{v\bot V} \frac{v^T XX^Tv}{v^Tv}\Big) = \frac{2}{n}\Big(\max_{v\bot V} \frac{v^T X'X'^Tv}{v^Tv}\Big) \le \frac{2}{n}\sigma_{\max}\big(X'\big)^2\numberthis\label{a-eqn:mainConstrintUpper}
\end{align*}
The last inequality follows from the Defn. of $\sigma_{\max}$. (Eqn. 5.3 in \cite{vershynin2010introduction}). We are now in a position to state our result.
  
\begin{theorem}
\label{a-thm:SDPGeneral}
Given a clustering instance $\mc X \subset \mb R^{N\times d}$ and $k$. Let $\mc X := \cup_{i=1}^k B_i$ where $B_i$ is a finite set of radius at most $r$ centered at $\mu_i$. That is, $B_i = \{x : \|x -\mu_i\| \le r\}$. Furthermore, let $\|\mu_i - \mu_j\| \ge \delta r$ and $ n := \min_{i} |B_i|$. Define $B_i' := \{x - \mu_i : x \in B_i\}.$  and $\mc X' = \cup B_i'$. If 
$$\delta > 1 + \sqrt{1+\frac{2\sigma_{\max}^2(\mc X')}{n}}$$ 
then the $k$-means SDP finds the intended cluster solution  $\mc C^* = \{B_1, \ldots, B_k\}$.
\end{theorem}

Note that Thm. \ref{a-thm:SDPGeneral} doesn't make any assumptions on the distribution that generated the points $\mc X$. In general, the eigenvalues of the matrix $\mc X'$ are bounded by $\sqrt{Nr}$. Let $\rho := \frac{N}{nk}$. $\rho$ measures the balance of the clusters. If all the clusters have the same size then $\rho = 1$. Using this notation, we get the following corollary
 
\begin{corollary}
\label{a-corollary:SDPGeneral}
Given a clustering instance $\mc X \subset \mb R^{N\times d}$ and $k$. Let $\mc X := \cup_{i=1}^k B_i$ where $B_i$ is a finite set of radius at most $r$ centered at $\mu_i$. That is, $B_i = \{x : \|x -\mu_i\| \le r\}$. Furthermore, let $\|\mu_i - \mu_j\| \ge \delta r$. If 
$$\delta > 1 + \sqrt{1+ 2r\rho k}$$ 
then the $k$-means SDP finds the intended cluster solution  $\mc C^* = \{B_1, \ldots, B_k\}$.
\end{corollary}

\noindent Cor. \ref{a-corollary:SDPGeneral} holds for any distribution that generated the data $\mc X$. However, the separation $\delta$ depends on the number of clusters $k$. Next, we show that if the data is generated by an isotropic distribution $\mc P$ then we can get rid of the dependence on $k$ and get the positive results as long as $\delta > 2$. 

In this setting, as before we are given a set $\mc X$ which can be covered by $k$ ``well-separated'' balls. That is, $\mc X := \cup_{i=1}^k B_i$ where each $B_i$ is generated as follows. Let $\mc P$ denote the isotropic distribution on the ball centered at origin of radius atmost $r$, that is $B_1(r)\subset \mb R^d$. Let $B_i$ be a set of $n_i$ points drawn according to $P_i$, the measure $\mc P$ translated to $\mu_i$. Also, $\|\mu_i - \mu_j\| > \delta r > 2r$. 



Let $\Theta = E[\|x_{pr}'\|^2]$. Using Thm. \ref{a-thm:spectralNormCOncentration}, we can bound the RHS of Eqn. \ref{a-eqn:mainConstraint} by upper bounding the maximum eigenvalue of $X'$ as 
\begin{align}
  &\mb P\bigg[\sigma_{\max}\bigg(\sqrt{\frac{d}{\theta}}X'\bigg) > \sqrt{N} + t\sqrt{\frac{d}{\theta}} \bigg] \le 2d\exp(-ct^2)
\end{align}

Now, let $t\sqrt{\frac{d}{\theta}} = s\sqrt{N}$. Then, we get that with probability atleast $1-2d\exp(-\frac{c\theta Ns^2}{d})$ we have that 
\begin{align*}
  &\frac{2}{n}\sigma_{\max}(X')^2 \le 2(1+s)^2\frac{N\theta}{nd} \le 2\rho\theta (1+s)^2\frac{1}{d}
\end{align*}

So we see that as long $(\delta - 1)^2 - 1 > 2k\rho\theta (1+s)^2\frac{1}{d}$, the primal and dual objective value are the same with high probability. In other words $\delta > 1 + \sqrt{1+2\rho\theta(1+s)^2\frac{k}{d}}$ implies the desired conditions. Now, we are finally ready to state our result.

\begin{theorem}
\label{a-thm:SDPIsometric}
Let $\mc P$ denote the isotropic distribution on the ball centered at origin of radius $r$, that is on, $B_1(r)$ in $\mb R^d$. Given centers $\mu_1, \ldots, \mu_k$ such that $\|\mu_i - \mu_j\| > \delta r > 2 r$. Let $\mc P_i$ be the measure $\mc P$ translated with respect to the center $\mu_i$. Given a clustering instance $\mc I := \cup_{i=1}^k B_i$ where each $B_i$ is drawn i.i.d w.r.t $\mc P_i$. If the distance between the centers of any two balls
$$\delta > 1 + \sqrt{1+\frac{2\theta\rho k}{d}\Big(1+\frac{1}{\log N}\Big)^2}$$  
where $\rho = \frac{N}{nk}$ and $ n := \min_{i\in[k]} |B_i|$ and $\theta = \mb E[\|x_{pi}-\mu_p\|] < 1$, then there exists a constant $c > 0$ such that with probability at least $1 - 2d\exp(\frac{-cN\theta}{d\log^2N})$ the $k$-means SDP finds the intended cluster solution  $\mc C^* = \{B_1, \ldots, B_k\}$.
\end{theorem}

\subsection{Tightness of the regularised SDP based algorithm}
\begin{equation*}
	\textbf{0-1 SDP} 
	\begin{cases}
		\min_{Z, y} \enspace &\tr(DZ) + \lambda \langle \mb 1, y\rangle\\
		\text{subject to } \enspace &\tr(Z) = k\\
		& Z\cdot \mb 1 + y = \mb 1\\	
		&Z\ge 0, Z^2 = Z, Z^T = Z \\
		& y \in \{0, 1\}^n
	\end{cases}
	\xrightarrow{\text{relaxed}} \textbf{ SDP } 
	\begin{cases}
		\min_{Z} \enspace &\tr(DZ) + \lambda \langle \mb 1, y\rangle\\
        \text{subject to } \enspace &\tr(Z) = k\\
		& \Big(\frac{Z+Z^T}{2}\Big)\cdot \mb 1 + y = \mb 1\\		
		&Z \ge 0, y \ge 0, Z \succeq 0 \numberthis\label{a-eqn:regularisedSDP}
	\end{cases}
\end{equation*}

We are given a set $\mc X \subset \mb R^d$. $\mc X = \mc I \cup \mc N$ is such that $\mc I$ can be covered by a set of $k$ ``well-separated'' balls. That is, $\mc I := \cup_{i=1}^k B_i$ where $B_i$ is a ball of radius at most $r$ centered at $\mu_i$ and $\|\mu_i - \mu_j\|_{2}^{2} \ge \delta r$. 
Define $n_i := |B_i|$ and $n := \min_{i\in[k]} n_i$ and $m := |\mc N| = n_{k+1}$ and $N = \sum_i n_i + m$. $D$ is an $N\times N$ matrix such that $D_{ij} = \|x_i -x_j\|^2$.  

Note that the clustering algorithm gets $\mc X$ as input and does not know about the sets $B_i$'s or $\mc I$ or $\mc N$. The goal is to output a clustering $\mc C^*$ of $\mc X$ such that $C^* = \{B_1, B_2, \ldots, B_k, \mc N\}$. From the way we constructed the 0-1 SDP, this corresponds to
\begin{align}
  Z^* = \sum_{p=1}^k \frac{\mb 1_p \mb 1_{p}^T}{n_p} \enspace\enspace\text{ and }\enspace\enspace y^* = \mb 1_{k+1}\label{a-eqn:regularisedIntendedsolution}
\end{align}
where $\mb 1_p$ is an $N$-dimensional indicator vector for the $p^{th}$ cluster. That is, $Z$ is a block diagonal matrix and consists of $k$ non-zero diagonal blocks. Observe that $Z$ consists of blocks $Z^{(p, q)}$. Also, for $1\le p \le k$, we have that $Z^{(p, p)} = \frac{1}{n_p}\mb 1\mb 1^T$ and for all $p \neq q, Z^{(p, q)} = 0$. Also, $Z^{(k+1, k+1)} = 0$. To prove that the regularised SDP (Eqn. \ref{a-eqn:regularisedSDP}) finds the desired solution, we  will adopt the following strategy. We first construct a dual for Eqn. \ref{a-eqn:regularisedSDP}. We then show that under certain conditions on $\delta$ (well-separateness of the balls) and $m$ (the number of noisy points) the following happens. The primal objective value and the dual objective value are the same. Also, the corresponding $Z$ satisfies Eqn. \ref{a-eqn:regularisedIntendedsolution}. 

Before, we describe the dual, lets introduce a bit of notation. We index every point as $(p, i)$ where $p$ denotes the ball (or cluster) to which it belongs and $i$ denotes the index within that ball. Observe that the distance matrix $D$ consists of blocks $D^{(p, q)}$ such that $D^{(p,q)}_{ij} = \|x_{(p,i)}-x_{(q,j)}\|^2_2$. Now, to construct the dual, we introduce variables $z, \alpha_{(p,i)}, \beta_{(p,i)(q,j)}$, $\gamma_{(p, i)}$ and $Q$ for each of the constraints in the primal problem. 
\[\textbf{SDP Dual}
    \begin{cases}
		\max \enspace &-zk - \sum_{(p, i)} \alpha_{(p,i)}\\
		\text{subject to } 
		&Q = D + zI + \sum_p\sum_i \alpha_{(p,i)}A_{(p, i)} -\sum_{p,q}\sum_{i,j} \beta_{(p,i)(q,j)}E_{(p,i)(q,j)}\\
		&\sum_{(p,i)} (\gamma_{(p,i)}-\alpha_{(q,i)}) e_{(p, i)} = \lambda \mb 1\\
		&\beta \ge 0, \gamma \ge 0 \\
		&Q \succeq 0
	\end{cases}
	\label{a-eqn:resularisedSDPDual}
	\numberthis
\]
where, $A_{(p,i)} = \frac{1}{2}(e_{p,i}\mb 1^T + \mb 1 e_{p, i}^T)$ and $E_{(p,i)(q,j)} = e_{(p,i)}e_{(q,j)}^T$. $\mb 1$ is an $N$-dimensional vector of all ones while $e_{(p,i)}$ is the indicator vector with one in position $(p, i)$ and zeros elsewhere. Now, we will examine the conditions under which the dual objective value matches the primal objective such that all the constraints of the dual are satisfied. 

\subsubsection*{Complementary slackness}
We know that $\beta_{(p, i)(q, j)} Z_{(p, i)(q, j)} = 0$. Now, for all $1 \le p \le k$, $Z^{(p, p) \neq 0}$ and for all the other pairs $(p, q)$ we have that $Z^{(p, q)} = 0$. Hence, we get that for all $1 \le p \le k$, $\beta^{(p, p)} = 0$. Also, we know that $\gamma_{(p, i)} y_{(p,i)} = 0$. Now, $y_{(k+1, i)} \neq 0$, hence $\gamma_{(k+1, i)} = 0$. Also, if we have that for all $1 \le k \le p$, $Q^{(p,p)}\mb1 = 0$ then $\langle Q, Z \rangle = 0$. This ensures that the second complementary slackness condition is also satisfied. 

\subsubsection*{Some properties of the $Q$ matrix}
Before we proceed, let's examine some properties of the dual matrix $Q$. Observe that for all $1 \le p\neq q \le k$, 

\begin{alignat*}{1}
  Q^{(p, q)} = 
  \begin{cases}
    &D^{(p,p)} + z \mb I + \frac{1}{2}\sum_{i=1}^{n_p} \alpha_{(p,i)} (e_{i}\mb 1^T + \mb 1e_{i}^T) \hspace{1in}\text{ if } 1 \le p = q \le k\\
    &D^{(k+1,k+1)} + z\mb I - \lambda \mb 1 \mb 1^T - \beta^{(p,q)} \hspace{1.45in}\text{ if } p = q = k+1\\
    &D^{(p,q)} + \frac{1}{2}\sum_{i=1}^{n_p} \alpha_{(p,i)} e_{i}\mb 1^T + \frac{1}{2}\sum_{i=1}^{n_q}\alpha_{(q,i)}\mb 1e_{i}^T - \beta^{(p,q)} \hspace{0.3in}\text{ otherwise }
  \end{cases}
  \label{a-eqn:resularisedQ}\numberthis
\end{alignat*} 

\subsubsection*{Dual matches intended primal solution}
Now, using the fact that for all $1\le p \le k$, $Q^{(p, p)}\mb 1 = 0$ implies that
\begin{align*}
  0 &= e_r^TD^{(p,p)}\mb1 + z + \frac{1}{2}\sum\alpha_{(p,i)} (e_r^Te_{i}\mb 1^T\mb1 + \mb e_r^T\mb1e_{i}^T\mb1)\\
  & = e_r^TD^{(p,p)}\mb1 + z + \frac{1}{2}\sum\alpha_{(p,i)} (e_r^Te_{i}n_p + 1) = e_r^TD^{(p,p)}\mb1 + z + \frac{1}{2}\sum\alpha_{(p,i)} +\frac{n_p\alpha_{(p,r)}}{2}.
\end{align*}
Summing over all $r$ and then substituiting back, we get that for all $1\le p \le k$ and for all $i$,  
\begin{align*}
  &\alpha_{(p, i)} = \frac{1^TD^{(p,p)}1}{n_p^2}-\frac{z}{n_p} -\frac{2e_i^TD^{(p,p)}1}{n_p} \\
  &= \frac{-2n_p^2 \|x_{pi}\|^2 + 4n_p \langle x_{pi}, \sum_jx_{pj}\rangle - 2\langle \sum_i x_{pi}, \sum_j x_{pj}\rangle -zn_p}{n_p^2} = -2\|x_{pi} - x_p\|^2-\frac{z}{n_p}
\end{align*}
Again, using complementary slackness, we know that $\gamma^{(k+1)} = 0$. This implies that $\alpha_{(k+1, i)} = -\lambda$. Combining these, we get that
\begin{align*}
  \alpha_{(p, i)} &= -2\|x_{pi} - x_p\|^2-\frac{z}{n_p} \\
  \alpha_{(k+1, i)} &= -\lambda \numberthis\label{a-eqn:regularisedAlpha} 
\end{align*}
Now, we have the value of $\alpha$ for all $1\le p \le k$ and for all $i$. Computing the objective function, we get that 
\begin{align*}
  kz + \sum_p\sum_i \alpha_{p,i} &= kz + \sum_{p}\frac{1^TD^{(p, p)}1}{n_p} - \sum_p z -2\sum_p \frac{1^TD^{(p,p)}1}{n_p} - \lambda m \\
  &= -\sum_p \frac{\mb1^TD^{(a,a)}\mb1}{n_p} - \lambda \langle \mb 1, y\rangle = -\langle D, Z \rangle - \lambda \langle \mb 1, y\rangle = - \tr(DZ) - \lambda \langle \mb 1, y\rangle
\end{align*}

\subsubsection*{Satisfying the $\lambda$ constraint of dual}
This constraint implies for all $1 \le p \le k, \gamma_{(p,r)} = \alpha_{(p,r)} + \lambda$. This will be satisfied as long as for all $p$ and for all $r$, $\lambda \ge -\alpha_{(p, r)}$. Choosing $\lambda$ as below ensures that the constraint is satisfied for all $p$ and all $r$.
\begin{align}
\lambda \ge \frac{z}{n} + 2\|x_{p,r} - x_p\|^2\label{a-eqn:lambdaLower}
\end{align}
where $x_p$ denotes the center of the $p^{th}$ cluster and $x_{p. r}$ denotes the $r^{th}$ point in the $p^{th}$ cluster. Hence, we see that for the intended solution, the primal and dual values are the same. Hence, solution is optimal. Now, the main question is to find $Q$ such that $Q$ is positive semi-definite while simultaneously ensuring that $\beta, \gamma \ge 0$.  

\subsubsection*{Satisfying PSD for $Q$}
Decompose $Q$ as follows.\[ Q = 
\begin{bmatrix}
    Q'    & B_1 \\
    B_2   & Q^{(k+1, k+1)}
\end{bmatrix}
\]
If $B_1 = B_2 = 0$ and $Q' \succeq 0$ and $Q^{(k+1, k+1)} \succeq 0$ then we know that $Q \succeq 0$. Let $X_1 = \{C_1, \ldots, C_k\}$ be the set of all points which were assigned to the $1 \le p \le k$ clusters. From the proof of the noiseless case, we know that if 
\begin{align*} \|x_{pr} - x_{q}\|^2 - \|x_{pr} - x_{p}\|^2 \ge \frac{z}{n} \ge \frac{2}{n}\Big(\max_{v\bot V} \frac{v^T X_1X_1^Tv}{v^Tv}\Big)  \numberthis\label{a-eqn:zLower}
\end{align*}
where $V = span\{1_p\}_{p=1}^k$  is the subspace spanned by $\mb 1_p$ (the indicator vector for the $p^{th}$ cluster) then $Q'$ is positive semi-definite. We know that the RHS is upper bounded by the square of maximum eigenvalue of $X_1$, which gives the following
\begin{align*} \|x_{pr} - x_{q}\|^2 - \|x_{pr} - x_{p}\|^2 \ge \frac{z}{n} \ge \frac{2}{n}\sigma_{\max}^2(X_1')  \numberthis\label{a-eqn:zConstraint}
\end{align*}
Hence, if either Eqn. \ref{a-eqn:zLower} or Eqn. \ref{a-eqn:zConstraint} can be satisfied then we $Q'$ is positive semi-definite. Here, the matrix $X_1'$ is such that $X_1 = X_1' + C$ where $C$ is a rank $k$ matrix which contains the centers $\mu_1, \ldots \mu_k$. Next, to ensure that $B_1 = 0$, we need that for all $1 \le p \le k, Q^{(p, k+1)} = 0$. Using Eqn. \ref{a-eqn:resularisedQ}, 
\begin{flalign*}
  0 &= Q^{(p, k+1)}_{rs} = D^{(p,k+1)}_{rs} + \frac{1}{2}\sum_i \alpha_{(p,i)} e_r^Te_{i}\mb 1^Te_s  - \frac{\lambda}{2}e_r^T\mb 1 \mb 1^Te_s - \beta^{(p,k+1)}_{rs}&\\
  & = D^{(p,k+1)}_{rs} - \|x_{pr} - x_{p}\|^2 - \frac{z}{2n_p} - \frac{\lambda}{2} - \beta^{(p,k+1)}_{rs}.\enspace \text{We need to choose $\beta^{(p, k+1)}_{rs} \ge 0$. Hence, }& \\
  &\implies \| x_{k+1, s} - x_{pr} \|^2 \enspace\ge\enspace \frac{\lambda}{2} + \frac{z}{2n_p} + \|x_{pr} - x_p\|^2&
\end{flalign*}
Thus, we see that if 
\begin{align*}
  &\| x_{k+1, s} - x_{pr} \|^2 \enspace\ge\enspace \frac{\lambda}{2} + \frac{z}{2n} + \|x_{pr} - x_p\|^2 \numberthis\label{a-eqn:lambdaUpper1}
\end{align*}
then for all $1 \le p \le k$, we have that $Q^{(p, k+1)} = 0$. In other words, we have that $B_1 = 0$. Next, to ensure that $B_2 = 0$, we need that for all $1 \le q \le k, Q^{(k+1, q)} = 0$. Using Eqn. \ref{a-eqn:resularisedQ}, 
\begin{flalign*}
  0 &= Q^{(k+1, q)}_{rs} = D^{(k+1,q)}_{rs} + \frac{1}{2}\sum_i \alpha_{(q,i)} e_r^Te_{i}\mb 1^Te_s  - \frac{\lambda}{2}e_r^T\mb 1 \mb 1^Te_s - \beta^{(k+1, q)}_{rs}&
\end{flalign*}
Using the same analysis as before, we see that if $\| x_{k+1, r} - x_{qs} \|^2 \enspace\ge\enspace \frac{\lambda}{2} + \frac{z}{2n} + \|x_{qs}-x_q\|^2$ then for all $1 \le q \le k$, we have that $Q^{(p, k+1)} = 0$. Observe that this is the same condition as Eqn. \ref{a-eqn:lambdaUpper1}. Thus, this ensures that $B_2 = 0$. Next, we need to show positive semi-definiteness of the matrix $Q^{(k+1, k+1)}$. Again, using Eqn. \ref{a-eqn:resularisedQ}, we get that for any vector $v \in \mb R^m$
\begin{flalign*}
  &v^TQ^{(k+1, k+1)}v = v^TD^{(k+1, k+1)}v - v^T\beta^{(k+1, k+1)}v + z v^T v - \lambda (v^T \mb 1)^2&
\end{flalign*}
To show that $Q^{(k+1, k+1)}$ is positive semi-definite, we need to ensure that the above is $\ge 0$ for all $v$. If we choose $\beta^{(k+1, k+1)} = D^{(k+1, k+1)}$ and 
\begin{align*}
  &\frac{z}{m} > \lambda\numberthis\label{a-eqn:lambdaUpper2}\\
  \text{then, we have that }\enspace &v^TQ^{(k+1, k+1)}v = z\sum_i v_i^2 - \lambda (\sum_i v_i)^2 \ge (z -\lambda m)\sum_i v_i^2 \ge 0
\end{align*}

\subsubsection*{Putting it all together}
We are given $X := I \cup N$. Let $I = \cup B_i$ where each $B_i$ is ball of radius atmost one and centered at $\mu_i$ where $\mu_i$ is the average of points in $B_i$. Also, $d(\mu_i, \mu_j) \ge \delta$. Decompose $N = N_1 \cup N_2$ into two sets. Let $N_2 = \{n \in N :$ for all $b \in I, d(n,b) \ge \nu\}$ and $N_1 = N \setminus N_2$. Let $N_1$ be such that for all $n \in N_1$, $| d^2(n, \mu_i) - d^2(n, \mu_j)| \ge \alpha$. 

We will show that the regularised SDP outputs the clustering $\mc C = \{C_1, \ldots, C_k, N_2\}$, where each $B_i \subseteq C_i$. Hence, the clusters contain all the points from the balls $B_i$ plus (maybe) points from the set $N_1$.  

Consider the $p^{th}$ cluster $C_p$. We know that $C_p = B_p \cup M_p$ where $M_p \subseteq N_1$. Now, $x_p = \frac{\sum_{x \in B_p} x + \sum_{n \in M_p} n}{|B_p|+|M_p|} = \frac{\mu_p |B_p| + |M_p|avg(M_p)}{|B_p| + |M_p|}$. Thus, we get that $d(x_p, \mu_p) = \frac{|M_p|}{|B_p| + |N_p|}d(\mu_p, avg(M_p)) \le \frac{|N_1|}{n}(\nu + 1) =: a$. Thus, we have that for all $x_{pr} \in B_p$
\begin{flalign*}
  &\|x_{pr} - x_q\|^2 - \|x_{pr} - x_p\|^2 \ge (\|x_{pr} - \mu_q\| - \|x_q - \mu_q\|)^2 - (\|x_{pr} - \mu_p\| + \|x_p - \mu_p\|)^2&\\
  &= \|x_{pr} - \mu_q\|^2 -\|x_{pr} - \mu_p\|^2 +  \|x_q - \mu_q\|^2 - \|x_p - \mu_p\|^2 - 2\|x_{pr} - \mu_p\|\|x_{p} - \mu_p\|&\\
  & - 2\|x_{pr}-\mu_q\|\|x_q-\mu_q\| \ge (\delta-1)^2 - 1 - a^2 - 4a&
\end{flalign*}
and for $x_{pr} \in M_p$
\begin{flalign*}
  &\|x_{pr} - x_q\|^2 - \|x_{pr} - x_p\|^2 \ge \alpha - a^2 - 4a&
\end{flalign*}
Choosing $\frac{z}{n} = (\delta-1)^2-1-a^2-4a > \frac{2}{n}\sigma_{\max}^2(X_1')$ and $\alpha - a^2 - 4a > \frac{2}{n}\sigma_{\max}^2(X_1')$ ensures that  Eqn. \ref{a-eqn:zConstraint} is satisfied. Next, we see that if $\lambda$ is such that
\begin{align*}
  & 2\|x_{k+1,r}-x_{qs}\|^2 - (\delta-1)^2 - 1 \ge \lambda \ge (\delta-1)^2 + 1 \numberthis\label{a-eqn:lambdaConstraint}
\end{align*}
then Eqns. \ref{a-eqn:lambdaLower} and Eqns. \ref{a-eqn:lambdaUpper1} can be satisfied. Hence, if $\nu \ge \sqrt{1 + (\delta-1)^2}$ then the above condition can be satisfied. Finally, combining Eqns. \ref{a-eqn:zConstraint} and \ref{a-eqn:lambdaConstraint}, we see that if 
\begin{align*}|N_2| \le n\frac{(\delta-1)^2-1-a^2-4a}{\lambda}
\end{align*}
then Eqn. \ref{a-eqn:lambdaUpper2} can also be satisfied. 

\begin{theorem}
\label{a-thm:regSDPGeneral}
Given a clustering instance $\mc X \subset \mb R^l$ and $k$. Let $\mc X : = \mc I \cup \mc N$. Let $\mc I := \cup_{i=1}^k B_i$ where $B_i$ is a ball of radius at most $1$ centered at $\mu_i$ and $d(\mu_i, \mu_j) \ge \delta $. Define $B_i' := \{x - \mu_i : x \in B_i\}$. Let $\mc N = \mc N_1 \cup N_2$ have the following properties. For all $n \in N_1$ and for all $i, j$, we have that $| d(n, \mu_i) - d(n, \mu_j)| \ge \alpha$. For all $n \in \mc N_2$ and for all $x \in I, d(n, x) \ge \nu \ge \sqrt{1 + (\delta-1)^2}$. Note that $\mc N_1 \cap \mc N_2 = \phi$. Let $n = \min_i |B_i|$ and $a = \frac{|\mc N_1|(\nu+1)}{n}$. If  

\begin{itemize}
  \item $\delta > 1 + \sqrt{1+ a^2 + 4a + \frac{2\sigma_{\max}^2(X_1')}{n}}$ where $ X_1' = \cup B_i' \cup N_1$.
  \item $\alpha > \sqrt{a^2 + 4a + \frac{2\sigma_{\max}^2(X_1')}{n}}$ 
  \item $\frac{|\mc N_2|}{n} \le \frac{(\delta-1)^2-1-a^2-4a}{\lambda}$
\end{itemize}
then the regularised $k$-means SDP finds the intended cluster solution  $\mc C^* = \{C_1, \ldots, C_k, \mc N_2\}$ where $B_i \subseteq C_i$ when given $\mc X$ and $2\nu^2 - (\delta-1)^2 - 1 \ge \lambda \ge (\delta-1)^2 + 1$ as input.
\end{theorem}

\noindent In general, the eigenvalues of the matrix $\mc X_1'$ are bounded by $\sqrt{|I| + |N_1|(\nu+1)}$. Let $\rho := \frac{|I|}{nk}$. $\rho$ measures the balance of the clusters. If all the clusters have the same size then $\rho = 1$. Hence, $\frac{2}{n}\sigma^2_{\max}(X_1') \le 2\rho k + a$. Assuming that $\frac{|\mc N_1|}{n} \le \epsilon$ and choosing $\nu = 2\delta - 1$, we get the following corollary.
 
\begin{corollary}
\label{a-corollary:SDPGeneral}
Given a clustering instance $\mc X \subset \mb R^l$ and $k$. Let $\mc X : = \mc I \cup \mc N$. Let $\mc I := \cup_{i=1}^k B_i$ where $B_i$ is a ball of radius at most $1$ centered at $\mu_i$ and $d(\mu_i, \mu_j) \ge \delta $. Define $B_i' := \{x - \mu_i : x \in B_i\}$. Let $\mc N = \mc N_1 \cup N_2$ have the following properties. For all $n \in N_1$ and for all $i, j$, we have that $| d(n, \mu_i) - d(n, \mu_j)| \ge \alpha$. For all $n \in \mc N_2$ and for all $x \in I, d(n, x) \ge 2\delta$. Note that $\mc N_1 \cap \mc N_2 = \phi$. Let $n = \min_i |B_i|$ and $\epsilon = \frac{|\mc N_1|}{n}$. If  

\begin{itemize}
  \item $\delta > \frac{1+5\epsilon}{1-4\epsilon^2} + \sqrt{ \Big(\frac{1+5\epsilon}{1-4\epsilon^2}\Big)^2 + \frac{2\rho k}{1-4\epsilon^2}}$  where $ X_1' = \cup B_i' \cup N_1$. 
  \item $\alpha \ge \sqrt{10\delta \epsilon + 4\delta^2 \epsilon^2+ 2\rho k}$ 
  \item $\frac{|\mc N_2|}{n} \le \frac{\delta^2(1-4\epsilon^2)-2\delta(1+4\epsilon)}{\lambda}$
\end{itemize}
then the regularised $k$-means SDP finds the intended cluster solution  $\mc C^* = \{C_1, \ldots, C_k, \mc N_2\}$ where $B_i \subseteq C_i$ when given $\mc X$ and $\delta^2+2\delta \ge \lambda \ge (\delta-1)^2 + 1$ as input.
\end{corollary}

\noindent Cor. \ref{a-corollary:SDPGeneral} holds for any distribution that generated the data $\mc X$. However, the separation $\delta$ depends on the number of clusters $k$. Next, we show that if the data is generated by an isotropic distribution $\mc P$ then we can get rid of the dependence on $k$ and get the positive results as long as $\delta > 2$. 

In this setting, as before the set $\mc I$ which can be covered by $k$ ``well-separated'' balls. That is, $\mc I := \cup_{i=1}^k B_i$ where each $B_i$ is generated as follows. Let $\mc P$ denote the isotropic distribution on the ball centered at origin of radius atmost $1$, that is $B_1(1)\subset \mb R^l$. Let $B_i$ be a set of $n_i$ points drawn according to $P_i$, the measure $\mc P$ translated to $\mu_i$. Also, $\|\mu_i - \mu_j\| > \delta$. We need to ensure that for $\mc X$ as generated above, Eqn. \ref{a-eqn:zLower} can be satisfied. 


We will now upper bound the we can bound the RHS of Eqn. \ref{a-eqn:zLower}. Decompose $X_1 := A_1' + C + N_1'$. Now, $A_1'$ contains points from the balls $B_1', \ldots, B_k'$. $C$ contains the centers $\mu_1, \ldots, \mu_k$ and $N_1'$ contains the points from the set $|N_1|$ but shifted by $\mu_p$. Now,
\begin{align*}
\frac{2}{n}\Big(\max_{v\bot V} \frac{v^T X_1X_1^Tv}{v^Tv}\Big) = \frac{2}{n}\Big(\max_{v\bot V} \frac{v^T A_1'A_1^Tv}{v^Tv}\Big) + \frac{2}{n}\Big(\max_{v\bot V} \frac{v^T N_1N_1^Tv}{v^Tv}\Big) \le \sigma_{\max}^2(A_1') + \sigma_{\max}^2(N_1')
\end{align*}
Now, it's easy to see that $\sigma_{\max}^2(N_1') \le |N_1|(\nu+1)$. Let $\theta = E[\|a_{pr}'\|^2]$. Using Thm. \ref{a-thm:spectralNormCOncentration}, we will upper bound the maximum eigenvalue of $A_1'$ as 
\begin{align}
  &\mb P\bigg[\sigma_{\max}\bigg(\sqrt{\frac{d}{\theta}}A'\bigg) > \sqrt{|I|} + t\sqrt{\frac{d}{\theta}} \bigg] \le 2d\exp(-ct^2)
\end{align}

Now, let $t\sqrt{\frac{d}{\theta}} = s\sqrt{|I|}$. Then, we get that with probability atleast $1-2d\exp(-\frac{c\theta |I|s^2}{d})$ we have that 
\begin{align*}
  &\frac{2}{n}\sigma_{\max}^2(A_1') \le 2(1+s)^2\frac{|I|\theta}{nd} \le 2k\rho\theta (1+s)^2\frac{1}{d}
\end{align*}

Thus, $(\delta-1)^2 -1 - a^2 - 4a > a + 2\rho k\theta(1+s)^2\frac{1}{d}$ implies the desired conditions. Now, we are finally ready to state our result.

\begin{theorem}
\label{a-thm:regSDPIsometric}
Given a clustering instance $\mc X \subset \mb R^l$ and $k$. Let $\mc X : = \mc I \cup \mc N$. Let $\mc I := \cup_{i=1}^k B_i$ where $B_i$ is a ball of radius at most $1$ centered at $\mu_i$ and $d(\mu_i, \mu_j) \ge \delta $. Define $B_i' := \{x - \mu_i : x \in B_i\}$. Let $\mc N = \mc N_1 \cup N_2$ have the following properties. For all $n \in N_1$ and for all $i, j$, we have that $| d(n, \mu_i) - d(n, \mu_j)| \ge \alpha$. For all $n \in \mc N_2$ and for all $x \in I, d(n, x) \ge 2\delta$. Note that $\mc N_1 \cap \mc N_2 = \phi$. Let $n = \min_i |B_i| \ge c_1\frac{\log(k/\delta)}{\gamma^2}$ and $\epsilon = \frac{|\mc N_1|}{n}$. If  

\begin{itemize}
  \item $\delta > \frac{1+5\epsilon}{1-4\epsilon^2} + \sqrt{ \Big(\frac{1+5\epsilon}{1-4\epsilon^2}\Big)^2 + \frac{2\rho k\theta(1+1/\log(|I|))^2}{d(1-4\epsilon^2)}}$ 
  \item $\alpha \ge \sqrt{10\delta \epsilon + 4\delta^2 \epsilon^2+ \frac{2\rho k\theta(1+1/\log|I|)^2}{d(1-4\epsilon^2)}}$ 
  \item $\frac{|\mc N_2|}{n} \le \frac{\delta^2(1-4\epsilon^2)-2\delta(1+4\epsilon)}{\lambda}$
\end{itemize}
then there exists a constant $c_2 > 0$ such that with probability at least $1-\delta - 2d\exp(\frac{-c_2|I|\theta}{d\log^2|I|})$ the regularised $k$-means SDP finds the intended cluster solution  $\mc C^* = \{C_1, \ldots, C_k, \mc N_2\}$ where $B_i \subseteq C_i$ when given $\mc X$ and $\delta^2+2\delta \ge \lambda \ge (\delta-1)^2 + 1$ as input.
\end{theorem}

\subsection{Tightness of the regularised LP based algorithm}
\begin{equation*}
  \begin{split}
	\textbf{IP}\\
	\textbf{}
  \end{split}
  \begin{cases}
		\min_{Z, y} \enspace &\sum_{pq}d^2(p, q)z_{pq} + \lambda \sum_p y_p\\
		\text{s.t. } \enspace &\sum_{p}z_{pp}= k\\
		& \sum_{q}z_{pq} + y_p = 1\\	
		&z_{pq} = z_{qp} \\
		& z_{pq} \in \{0, z_{pp}\}, y \in \{0, 1\}^n
	\end{cases}
	\xrightarrow{\text{relaxed}} \textbf{ LP } 
	\begin{cases}
		\min_{Z, y} \enspace &\sum_{pq}d^2(p, q)z_{pq} + \lambda \sum_p y_p\\
        \text{s.t. } &\sum_{q}z_{pq} + y_p = 1\\
		& z_{pq} \le z_{pp}\\
		& \sum_{p} z_{pp} = k\\		
		& z_{pq} \ge 0, y_p \ge 0 \numberthis\label{a-eqn:regularisedLP}
	\end{cases}
\end{equation*}

We are given a set $\mc X \subset \mb R^d$. $\mc X = \mc I \cup \mc N$ is such that $\mc I$ can be covered by a set of $k$ ``well-separated'' balls. That is, $\mc I := \cup_{i=1}^k B_i$ where $B_i$ is a ball of radius at most $r$ centered at $\mu_i$ and $\|\mu_i - \mu_j\|_{2}^{2} \ge \delta r$. Let $\mc P$ denote the isotropic distribution in the unit ball centered at origin $B_1(r)$ in $\mb R^d$. The ball $B_i$ is drawn from the isotropic distribution $\mc P_i$, that is, the measure $\mc P$ translated with respect to the center $\mu_i$. Define $n_i := |B_i|$ and $n := \min_{i\in[k]} n_i$ and $m := |\mc N| = n_{k+1}$ and $N = \sum_i n_i + m$. $D$ is an $N\times N$ matrix such that $D_{ij} = \|x_i -x_j\|^2$.

Note that the clustering algorithm gets $\mc X$ as input and does not know about the sets $B_i$'s or $\mc I$ or $\mc N$. The goal is to output a clustering $\mc C^*$ of $\mc X$ such that $C^* = \{B_1, B_2, \ldots, B_k, \mc N\}$. From the way we constructed the Integer program, this corresponds to
\begin{align}
  Z^* = \sum_{p=1}^k \frac{\mb 1_p \mb 1_{p}^T}{n_p} \enspace\enspace\text{ and }\enspace\enspace y^* = \mb 1_{k+1}\label{a-eqn:regularisedIntendedsolution}
\end{align}
where $\mb 1_p$ is an $N$-dimensional indicator vector for the $p^{th}$ cluster. That is, $Z$ is a block diagonal matrix and consists of $k$ non-zero diagonal blocks. Observe that $Z$ consists of blocks $Z^{(p, q)}$. Also, for $1\le p \le k$, we have that $Z^{(p, p)} = \frac{1}{n_p}\mb 1\mb 1^T$ and for all $p \neq q, Z^{(p, q)} = 0$. Also, $Z^{(k+1, k+1)} = 0$. To prove that the regularised LP (Eqn. \ref{a-eqn:regularisedLP}) finds the desired solution, we  will adopt the following strategy. We first construct a dual for Eqn. \ref{a-eqn:regularisedLP}. We then show that under certain conditions on $\delta$ (well-separateness of the balls) and $m$ (the number of noisy points) the following happens. The primal objective value and the dual objective value are the same. Also, the corresponding $Z$ and $y$ satisfy Eqn. \ref{a-eqn:regularisedIntendedsolution}.

Now to construct the dual, we introduce variables $\alpha_p, \beta_{pq}, \gamma$, $\mu_{pq}$ and $\eta_p$ for each of the constraints in the primal problem. 
\[\textbf{LP Dual}
    \begin{cases}
		\max \enspace & \sum_{p} \alpha_{p} -\gamma k\\
		\text{subject to } 
		&\alpha_p + \mu_{pq} = \beta_{pq} + d^2(p, q)\\
		&\gamma = \sum_q \beta_{pq}\\
		&\lambda = \alpha_p + \eta_p \\
		&\mu_{pq} \ge 0, \eta_p \ge 0
	\end{cases}
	\label{a-eqn:resularisedLPDual}
	\numberthis
\]
Now, we will examine the conditions under which the dual objective value matches the primal objective such that all the constraints of the dual are satisfied. Before we go into more details, lets introduce the following notation. We will refer to points using symbols $a$, $a'$, $b$, $b'$, $c$ and $c'$. $a, b$ denotes two points in different clusters. $a$, $a'$ and $b, b'$ will refer to a pair of points in the same cluster. $c$ and $c'$ refers to the points in the noisy cluster $k+1$. 

\subsubsection*{Complementary slackness}
\begin{itemize}
  \item $\mu_{pq}z_{pq} = 0$. We get that
  \begin{align}
    &\mu_{aa'} = \mu_{bb'} = 0 \label{eqn:muCS}
  \end{align}
  \item $\eta_{p}y_{p} = 0$. We get that
  \begin{align}
    &\eta_{c} = 0 \label{a-eqn:etaCS}
  \end{align}
  \item $\beta_{pq}(z_{pq}-z_{pp}) = 0$. We get that
  \begin{align}
    &\beta_{ab} = \beta_{ac} = \beta_{ba} = \beta_{bc} = 0 \label{a-eqn:betaCS}
  \end{align}  
\end{itemize}

\subsubsection*{Dual matches intended primal solution}
Let $a \in C_i$ where $1 \le i \le k$. Hence, we get that 
\begin{align*}
  &\sum_{a' \in C_i}\alpha_{a} + \sum_{a'\in C_i}\mu_{aa'} = \sum_{a' \in C_i}\beta_{aa'}+ \sum_{a' \in C_i}d^2(a, a')\\
  \implies&\alpha_a = \frac{\gamma}{n_i} + \frac{\sum_{a'\in C_i}d^2(a, a')}{n_i}\numberthis\label{a-eqn:alphaA}
\end{align*}
For $c \in C_{k+1}$, we have that
\begin{align*}
  &\alpha_c + \eta_c= \lambda \implies \alpha_c = \lambda.
\end{align*}
Using these two equations, we get that
\begin{align*}
  \sum_p \alpha_p - k\gamma &= \sum_{i=1}^k \sum_{a \in C_i} \alpha_i + \sum_{c \in C_{k+1}}\alpha_c - k\gamma = \sum_{i=1}^k \sum_{a \in C_i} \frac{\gamma}{n_i} + \frac{\sum_{a'\in C_i}d^2(a, a')}{n_i} + \lambda \langle \mb 1, y\rangle - k\gamma\\
  &= \sum_{i=1}^k \gamma + \sum_{i=1}^k \frac{1}{n_i}\sum_{a, a' \in C_i}d^2(a, a') + \lambda \langle \mb 1, y\rangle - k\gamma = \sum_{pq} d^2(p, q)z^*_{pq} + \lambda\sum_{p}y^*_{p}
\end{align*}

\subsubsection*{Satisfying the $\lambda$ constraint of dual}
We have already seen that $\alpha_c$ should be equal to $\lambda$. Hence, if $\lambda \ge \alpha_a$ for all $a \in C_1, \ldots, C_k$ then this constraint can be satisfied. Thus, we get that if
\begin{align*}
  & \alpha_a \le \lambda \enspace\text{ and }\enspace\alpha_c = \lambda \numberthis\label{a-eqn:lambda}
\end{align*}
then the $\lambda$ constraint of the dual can be satisfied.
\subsubsection*{Satisfying the $\alpha_p$ constraint}
Again observe that for $a \in C_i$ and $b \in C_j \not=C_i$ and $c \in C_{k+1}$
\begin{align*}
  &\alpha_a + \mu_{ab} = d^2(a, b) \implies \alpha_a \le d^2(a, b)\\
  &\alpha_a + \mu_{ac} = d^2(a, c) \implies \alpha_a \le d^2(a, c)\numberthis\label{a-eqn:alphaUpper}
\end{align*}
To satisfy the constraint for $\alpha_c$, we need that 
\begin{align*}
  &\alpha_c + \mu_{cq} = \beta_{cq} + d^2(c, q) \implies \alpha_c |X| + \sum_q\mu_{cq} = \gamma + \sum_q d^2(c, q)
\end{align*}
This can be satisfied as long as 
\begin{align}
  &\lambda |X| \le \gamma + \sum_q d^2(c, q) \label{a-eqn:lambdaUpper}
\end{align}

\subsubsection*{Putting it all together}
From Eqns. \ref{a-eqn:alphaUpper}, \ref{a-eqn:lambda}, \ref{a-eqn:alphaA} and \ref{a-eqn:lambdaUpper}, we see that the following constraints need to be satisfied. Let $a, a' \in C_a$, $b \not\in C_a \cup C_{k+1}$ and $c 
\in C_{k+1}$.
\begin{align}
  d^2(a, a') &\le \frac{\gamma}{n_i} + \frac{\sum_{a_1 \in C_i}d^2(a, a_1)}{n_i} \le d^2(a, b) \label{a-eqn:dab}\\
  d^2(a, a') &\le \frac{\gamma}{n_i} + \frac{\sum_{a_1 \in C_i}d^2(a, a_1)}{n_i} \le d^2(a, c) \label{a-eqn:dac}\\
  \frac{\gamma}{n_i} + \frac{\sum_{a_1 \in C_i}d^2(a, a_1)}{n_i} &\le \lambda \le \frac{\gamma}{|X|} + \sum_q \frac{d^2(c, q)}{|X|} \label{a-eqn:lambdaUL}
\end{align}

Following the exact same analysis as in \cite{awasthi2015relax}, we know that Eqn. \ref{a-eqn:dab} can be satisfied with high probability (as the points in the balls $B_i$ are generated by an isotropic distribution).   Let $\nu$ denote the minimum distance between any point in $C_{k+1}$ to any other point in $C_1, \ldots, C_k$. Choosing $\nu \ge (\delta-2)$ ensures that Eqn. \ref{a-eqn:dac} is satisfied. Furthermore, if the number of noisy points $m \le N (1-\frac{4}{\nu^2})$ then Eqn. \ref{a-eqn:lambdaUL} can be satisfied.

\begin{theorem}
\label{a-theorem:lptight}
Given a clustering instance $\mc X \subset \mb R^d$ and $k$. Let $\mc X : = \mc I \cup \mc N$. Let $\mc I := \cup_{i=1}^k B_i$ where $B_i$ is a ball of radius at most $r$ centered at $\mu_i$ and $\|\mu_i - \mu_j\| \ge \delta r$. Let $\mc P$ denote the isotropic distribution on the ball centered at origin of radius $r$, that is on, $B_1(r)$ in $\mb R^d$. The ball $B_i$ is drawn from the isotropic distribution $\mc P_i$, that is, the measure $\mc P$ translated with respect to the center $\mu_i$.  Let $\mc N$ have the following property. Each $p_n \in \mc N$ is such that $\min_{p_i \in \mc I} \|p_n - p_i\| \ge \nu r$. Let $|\mc N| =: m$. If

\begin{itemize}
  \item $\delta > 4$ and $\nu > \delta - 2$ 
  \item $m \le |\mc I|\Big(\frac{\nu^2}{(\delta-2)^2}-1\Big)$
\end{itemize}
then the regularised $k$-means LP finds the intended cluster solution  $\mc C^* = \{B_1, \ldots, B_k, \mc N\}$ when given $\mc X$ and $(\delta - 2)^2r^2 \le \lambda \le \nu^2(1-\frac{m}{N})$ as input.
\end{theorem}

\begin{figure}[t]
  \label{a-figure:simulation}
  \centering
  \includegraphics[width=\textwidth]{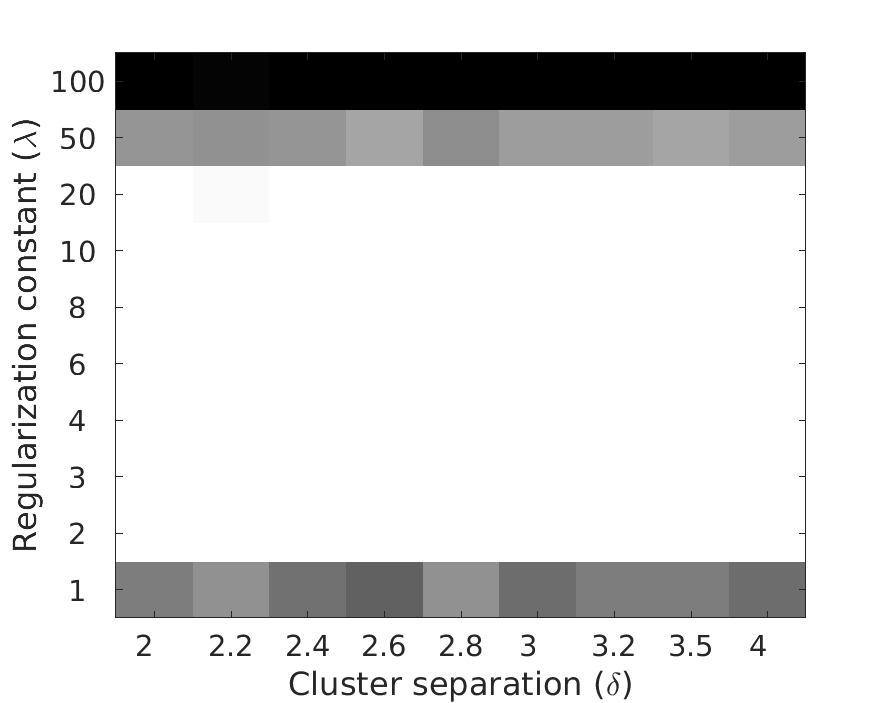}
  \includegraphics[width=\textwidth]{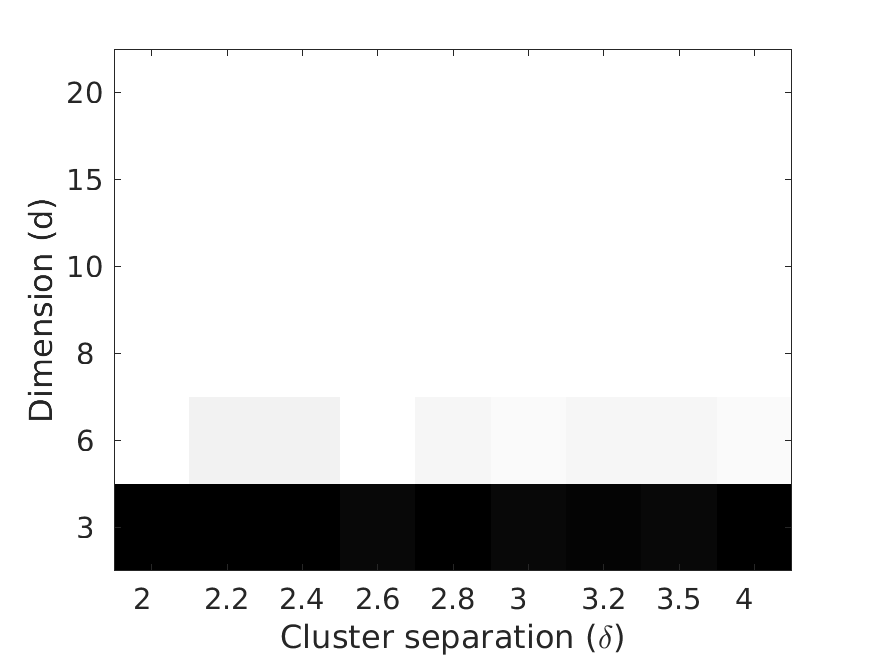}
  \caption{Heatmap showing the probability of success of the $k$-means regularised sdp algorithm. Lighter color indicates probability closer to one while darker indicates probability closer to zero.}
\end{figure}
\begin{figure}
  \includegraphics[width=\textwidth]{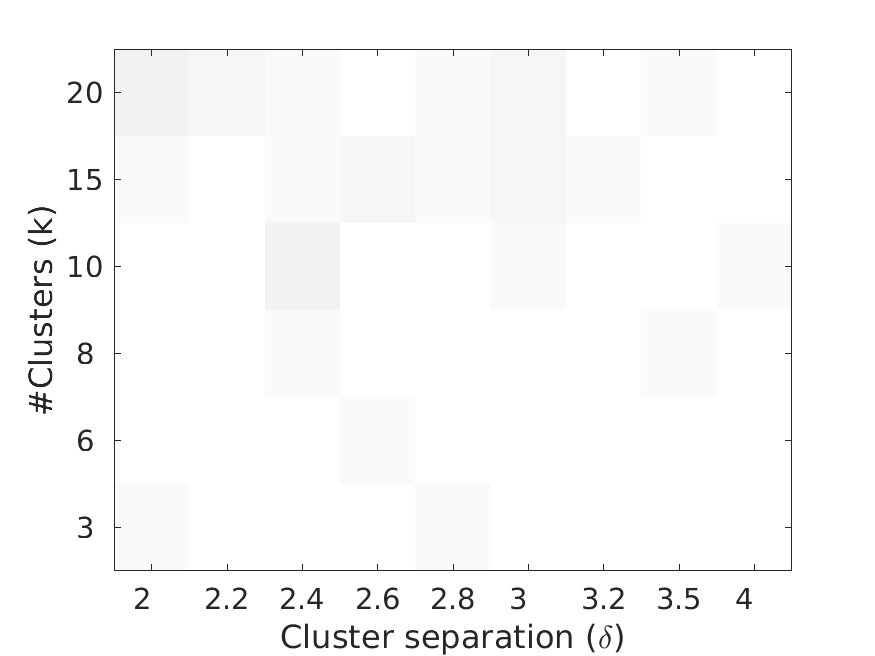}
  \includegraphics[width=\textwidth]{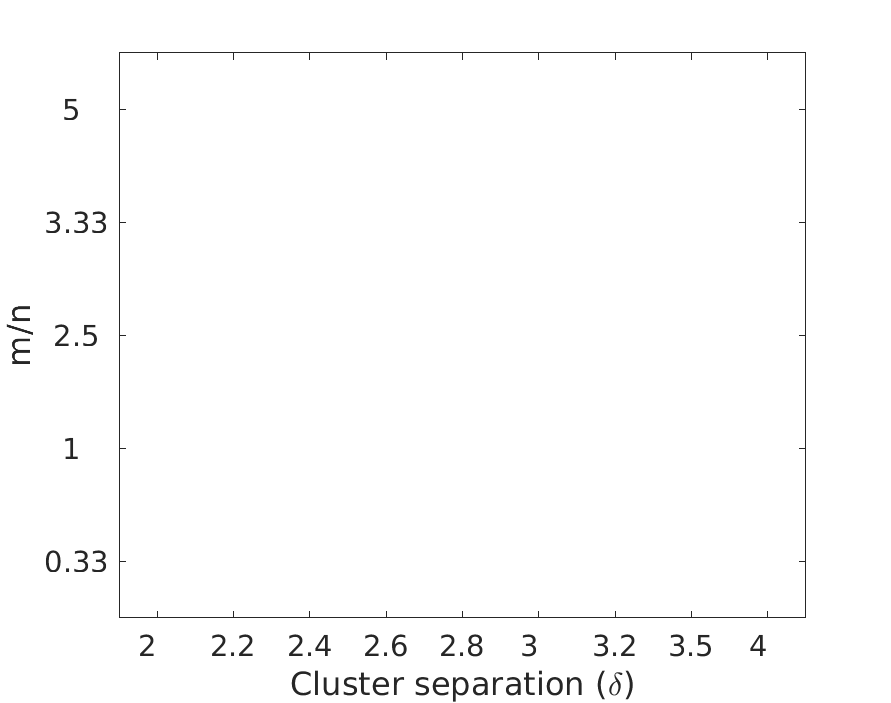}
  \caption{Heatmap showing the probability of success of the $k$-means regularised sdp algorithm. Lighter color indicates probability closer to one while darker indicates probability closer to zero.$m$ denotes the number of noisy points while $n$ denotes the number of points in the smallest cluster.}
\end{figure}
\begin{figure}
  \includegraphics[width=\textwidth]{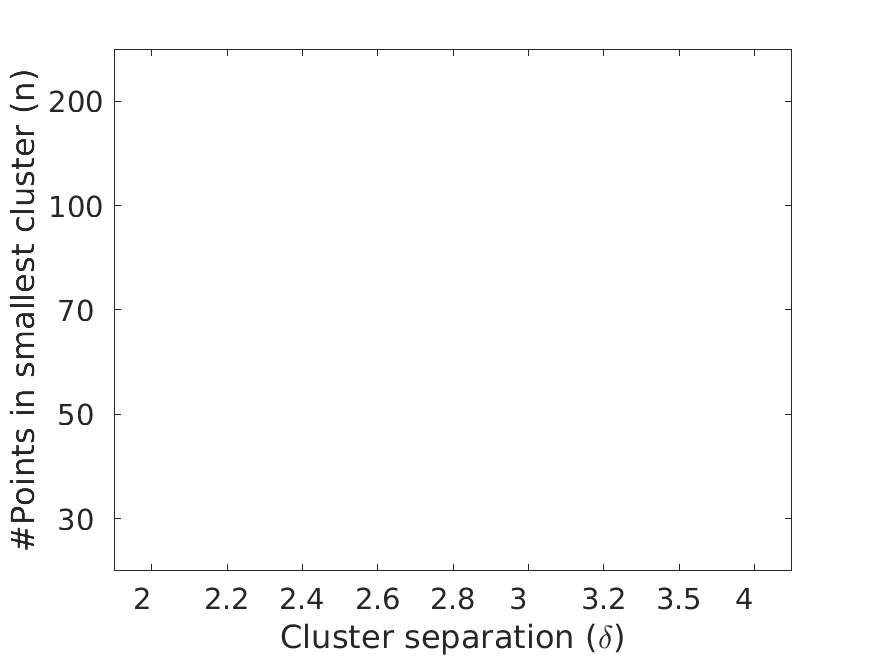}
  \includegraphics[width=\textwidth]{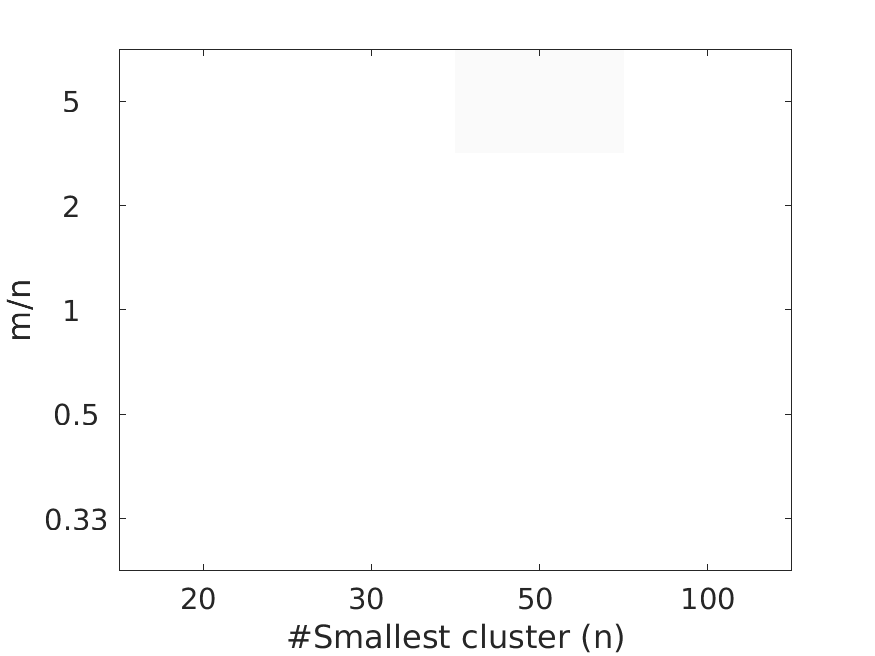}
  \caption{Heatmap showing the probability of success of the $k$-means regularised sdp algorithm. Lighter color indicates probability closer to one while darker indicates probability closer to zero. $m$ denotes the number of noisy points while $n$ denotes the number of points in the smallest cluster.}  
\end{figure}

\section{Technical lemma}
\begin{theorem}[Thm. 5.41 in \cite{vershynin2010introduction}]
\label{a-thm:spectralNormCOncentration}
Let $A$ be an $N\times d$ matrix whose rows $A_i$ are independent isotropic random vectors in $\mb R^d$. Let $m$ be a number such that $\|A_i\| \le \sqrt{m}$ almost surely for all $i$. Then for every $t$, one has
$$\sqrt{N} - t\sqrt{m} \le \sigma_{\min}(A) \le \sigma_{\max}(A) \le \sqrt{N} + t\sqrt{m}$$
with probability atleast $1-2d\exp(-ct^2)$, where $c$ is an absolute constant. $\sigma_{\min}$ and $\sigma_{\max}$ are the spectral norms or the minimum and maximum eigenvalues respectively of the matrix $A$.
\end{theorem}
\end{document}